\documentclass{article}

\usepackage{microtype}  % microtypography
\usepackage{graphicx}
\usepackage{booktabs} % for professional tables

\usepackage[final]{hyperref}

\usepackage[utf8]{inputenc}             % allow utf-8 input
\usepackage[T1]{fontenc}                % use 8-bit T1 fonts
\usepackage{url}                        % simple URL typesetting
\usepackage{amsmath,amsfonts,amssymb}   % blackboard math symbols
\usepackage{amsthm}
\usepackage{nicefrac}                   % compact symbols for 1/2, etc.
\usepackage{xcolor}                     % colors
\usepackage{graphicx}                   % graphics
\usepackage[inline]{enumitem}
\usepackage{caption}
\usepackage{subcaption}
\usepackage{wrapfig}

\usepackage{tikz,array}
\usetikzlibrary{shapes.geometric}
\usetikzlibrary{shapes.arrows}
\usetikzlibrary{calc}
\usetikzlibrary{intersections}

\usepackage[accepted]{icml2024}

\usepackage[capitalize,noabbrev]{cleveref}

\newtheorem{theorem}{Theorem}[section]
\newtheorem{definition}[theorem]{Definition}
\newtheorem{lemma}[theorem]{Lemma}
\newtheorem{proposition}[theorem]{Proposition}
\newtheorem{assumption}{Assumption}[section]

\newcommand{\bR}{\mathbb{R}}    % bb R
\newcommand{\bE}{\mathbb{E}}    % bb E
\newcommand{\cN}{\mathcal{N}}    % cal N
\newcommand{\cL}{\mathcal{L}}    % cal L

\newcommand{\cM}{\mathcal{M}}

\newcommand{\relgrad}{\Psi}
\newcommand{\grad}{{\mathrm{grad}}}

\newcommand{\stiefelG}[1]{{\mathrm{St}_{#1}(p, n)}}
\newcommand{\stiefelGeps}[1]{{\mathrm{St}^\varepsilon_{#1}(p, n)}}
\newcommand{\lagrM}{\lambda}
\newcommand{\rD}{\mathrm{D}}

\newcommand{\normCoef}{\omega}
\newcommand{\Retr}{\mathrm{Retr}}

\newcommand{\constH}{C_h}
\newcommand{\constHlow}{\bar C_h}
\newcommand{\constRgrad}{C_\relgrad}
\newcommand{\constLagrM}{C_\lambda}

\newcommand{\var}{\gamma}

\newcommand{\sk}{{\mathrm{skew}}}
\newcommand{\sym}{{\mathrm{sym}}}

\newcommand{\inner}[2]{\left\langle#1,\,#2\right\rangle} %nabla fx
\newcommand{\bigO}{\mathcal{O}} 
\newcommand{\tol}{e}

\DeclareMathOperator{\Tr}{Tr} % Trace

\newcommand{\Rev}[1]{{#1}} % command for revision, internal highlight

\icmltitlerunning{Optimization without Retraction on the Random Generalized Stiefel Manifold}

\begin{document}

\twocolumn[
\icmltitle{Optimization without Retraction on the Random Generalized Stiefel Manifold}

% It is OKAY to include author information, even for blind
% submissions: the style file will automatically remove it for you
% unless you've provided the [accepted] option to the icml2024
% package.

% List of affiliations: The first argument should be a (short)
% identifier you will use later to specify author affiliations
% Academic affiliations should list Department, University, City, Region, Country
% Industry affiliations should list Company, City, Region, Country

% You can specify symbols, otherwise they are numbered in order.
% Ideally, you should not use this facility. Affiliations will be numbered
% in order of appearance and this is the preferred way.
\icmlsetsymbol{equal}{*}

\begin{icmlauthorlist}
\icmlauthor{Simon Vary}{uclouvain,ox}
\icmlauthor{Pierre Ablin}{apple}
\icmlauthor{Bin Gao}{cas}
\icmlauthor{P.-A. Absil}{uclouvain}
%\icmlauthor{}{sch}
%\icmlauthor{}{sch}
\end{icmlauthorlist}

\icmlaffiliation{uclouvain}{ICTEAM Institute, UCLouvain, Louvain-la-neuve, Belgium}
\icmlaffiliation{ox}{Department of Statistics, University of Oxford, Oxford, United Kingdom}
\icmlaffiliation{apple}{Apple Machine Learning Group, Paris, France}
\icmlaffiliation{cas}{Academy of Mathematics and Systems Science, Chinese Academy of Sciences, Beijing, China}

\icmlcorrespondingauthor{Simon Vary}{simon.vary@stats.ox.ac.uk}
%\icmlcorrespondingauthor{Firstname2 Lastname2}{first2.last2@www.uk}

% You may provide any keywords that you
% find helpful for describing your paper; these are used to populate
% the "keywords" metadata in the PDF but will not be shown in the document
\icmlkeywords{Optimization, Manifold Constraints, Stochastic optimization, CCA, Generalized Eigenvalue Problem, Generalized Stiefel Manifold}

\vskip 0.3in
]

% this must go after the closing bracket ] following \twocolumn[ ...

% This command actually creates the footnote in the first column
% listing the affiliations and the copyright notice.
% The command takes one argument, which is text to display at the start of the footnote.
% The \icmlEqualContribution command is standard text for equal contribution.
% Remove it (just {}) if you do not need this facility.

\printAffiliationsAndNotice{}  % leave blank if no need to mention equal contribution
%\printAffiliationsAndNotice{\icmlEqualContribution} % otherwise use the standard text.

\begin{abstract}
    Optimization over the set of matrices $X$ that satisfy $X^\top B X = I_p$, referred to as the generalized Stiefel manifold, appears in many applications involving sampled covariance matrices such as the canonical correlation analysis (CCA), independent component analysis (ICA), and the generalized eigenvalue problem (GEVP). Solving these problems is typically done by iterative methods that require a fully formed $B$. We propose a cheap stochastic iterative method that solves the optimization problem while having access only to random estimates of $B$. Our method does not enforce the constraint in every iteration; instead, it produces iterations that converge to critical points on the generalized Stiefel manifold defined in expectation. The method has lower per-iteration cost, requires only matrix multiplications, and has the same convergence rates as its Riemannian optimization counterparts that require the full matrix $B$. Experiments demonstrate its effectiveness in various machine learning applications involving generalized orthogonality constraints, including CCA, ICA, and the GEVP.
\end{abstract}

\section{Introduction\label{sec:introduction}}
Many problems in machine learning and engineering, including canonical correlation analysis (CCA) \citep{Hotelling1936Relations}, independent component analysis (ICA) \citep{Comon1994Independent}, linear discriminant analysis \citep{McLachlan1992Discriminant}, and the generalized eigenvalue problem (GEVP) \citep{Saad2011Numerical}, can be formulated as the following optimization problem: 
\begin{equation}
    \begin{gathered}
    \min_{X\in \stiefelG{B}} f(X):=\bE[f_{\xi}(X)],\quad\mathrm{s.\,t.}\quad B = \mathbb{E}[B_\zeta], \\
   \stiefelG{B} := \left\{X\in\bR^{n\times p} \mid X^\top B X=I_p \right\} \label{eq:optimization_generalized_stiefel}
    \end{gathered}
\end{equation}
where the objective function $f$ is the expectation of $L$-smooth functions $f_\xi$, $B\in\mathbb{R}^{n\times n}$ is a positive-definite matrix, and $\xi, \zeta$ are independent~random variables. The individual random matrices $B_\zeta$ are only assumed to be positive semidefinite. The feasible set $\stiefelG{B}\subset\bR^{n\times p}$ defines a smooth manifold referred to as the \emph{generalized Stiefel manifold}.

In the deterministic case, when we have access to the matrix~$B$, the optimization problem can be solved by Riemannian techniques \citep{Absil2008Optimization, Boumal2023introduction}. Riemannian methods produce a sequence of iterates belonging to the set $\stiefelG{B}$, often by repeatedly applying a \emph{retraction} that maps tangent vectors to points on the manifold. In the case of $\stiefelG{B}$, retractions require non-trivial linear algebra operations such as eigenvalue or Cholesky decomposition. % and provided that $B$ is of full rank such that the projection on $\stiefelG{B}$ is tractable \citep{Absil2008Optimization, Boumal2023introduction}.
%On the other hand, optimization on $St_B(p,n)$$ also lends itself to infeasible optimization methods, such as the augmented Lagrangian method. Such methods are typically employed ...
On the other hand, optimization on $\stiefelG{B}$ also lends itself to infeasible optimization methods, such as the augmented Lagrangian method. Such methods are typically employed in deterministic setting when the feasible set does not have a convenient projection, e.g., it lacks a closed-form expression or it requires solving an expensive optimization problem~\citep{Bertsekas1982Constrained}. Infeasible approaches produce iterates that do not belong to the feasible set but converge to it by solving a sequence of unconstrained optimization problems. However, solving the optimization subproblems in each iteration might be computationally expensive and the methods are sensitive to the choice of hyper-parameters, both in theory and in practice.

In this paper, unlike in the aforementioned areas of study, we consider the setting \eqref{eq:optimization_generalized_stiefel} where the feasible set itself is \emph{stochastic}, i.e.,\ the matrix $B$ is unknown and is an expectation of random estimates $B_\zeta$, for which neither Riemannian methods nor infeasible optimization techniques are well-suited. In particular, we are interested in the case where we only have access to i.i.d.~samples from $\xi$ and $\zeta$, and not to the full function $f$ and matrix $B$. 

\begin{figure}[t]
    \tiny
    \centering
    \begin{tikzpicture}[scale=.55]
	% color
	\definecolor{tikzblue}{RGB}{74,125,179}
	\definecolor{tikzred}{RGB}{209,53,43}
		
	% safe region
	\fill[gray!10] (5,-2) -- (7,2) -- (7,2) arc (60:120:14) -- (-7,2) -- (-5,-2) -- (5,-2) arc (60:120:10);
	\fill[white] (5,-2) -- (5,-2) arc (60:120:10) -- (5,-2);
		
	% landing field
	\coordinate [fill=black,inner sep=.8pt,circle,label=180:{$X$}] (X) at (-2.5,3.3);
	\coordinate [label=0:{\color{tikzblue}$-\Psi_{\xi, \zeta, \zeta'}$}] (G) at (-0.55,3.56);
	\coordinate [label=180:{\color{tikzred}$-\omega\nabla\mathcal{N}_{\zeta, \zeta'}$}] (N) at (-2.3,2.2);
	\coordinate [label=0:{$-\Lambda_{\xi, \zeta, \zeta'}$}] (L) at ($(G)-(X)+(N)-(X)+(X)$);
				
	\draw[-{stealth[black]},black,thick] (X) --  (L);
	\draw[-{stealth[tikzblue]},tikzblue,thick] (X) --  (G);
	\draw[-{stealth[tikzred]},tikzred,thick] (X) --  (N);
	\draw[dash dot] (G) -- (L);
	\draw[dash dot] (N) -- (L);
        % \draw ($(X)+0.15*(N)-0.15*(X)$) -- ($(X)+0.15*(L)-0.15*(X)$) -- ($(X)+0.15*(G)-0.15*(X)$);
				
%		\draw [help lines,step=1] (-6,-5) grid (6,5);

	% noise
        \coordinate [label=180:{$\mathrm{St}_{B_\zeta}(p,n)$}] (noise) at (-6.2,0.6);
	\draw[dash dot dot,thick] (-6.2,0.4) .. controls (-4,-1) and (-3,3.5) .. (-1,1);
	\draw[dash dot dot,thick] (-1,1) .. controls (0,0) and (1,3) .. (3,1);
	\draw[dash dot dot,thick] (3,1) .. controls (4,-0.2) and (5,4) .. (5.8,-0.4);
		
	% manifold
	\coordinate [label=-135:{$\mathrm{St}_B(p,n)$}] (Stiefel) at (-5.8,0.2);
	\coordinate [label=180:{\color{gray}$\stiefelGeps{B}$}] (eps) at (5.5,-0.8);
	\draw[thick] (6,0) arc (60:120:12);
	\draw[dashed] (7,2) arc (60:120:14);
	\draw[dashed] (5,-2) arc (60:120:10);
    \end{tikzpicture}
    \caption{Illustration of the landing field and the random feasible~set. \label{fig:diagram_landing}}
\end{figure}
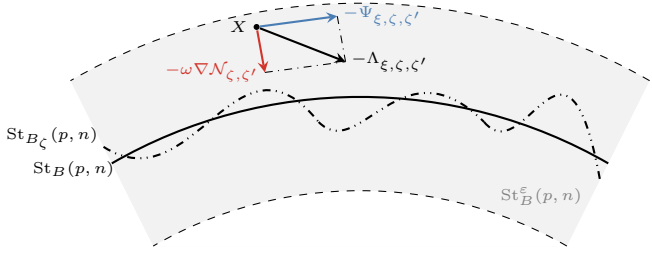

We design an iterative \emph{landing} method requiring only random estimates $B_\zeta$ that provably converges to critical points of \eqref{eq:optimization_generalized_stiefel} while performing only matrix multiplications. The main principle of the method is depicted in the diagram in Figure~\ref{fig:diagram_landing}. It is inspired by a recent line of work that first considered the orthogonal group $\mathrm{St}_{I_n}(n,n)$~\cite{Ablin2022Fast} and was later extended to the Stiefel manifold $\mathrm{St}_{I_n}(p,n)$~\cite{Gao2022Optimization,Ablin2023Infeasible,Schechtman2023Orthogonal}. Instead of performing retractions after each iteration, the proposed algorithm performs an update along the sum of two orthogonal 
vectors---one is an unbiased estimator of a \emph{relative ascent direction} (a concept defined in Section~\ref{sec:landing}) and the other is an unbiased estimator of a direction towards $\mathrm{St}_B(p,n)$. The algorithm does not enforce the constraint in every iteration; instead it produces iterations that remain within an initially prescribed $\varepsilon$-safe region, and finally ``lands'' on, i.e.,~converges to, the manifold. 

Specifically, the proposed stochastic landing iteration for solving \eqref{eq:optimization_generalized_stiefel} is the simple, cheap, and stochastic update rule 
\begin{equation}
    \begin{gathered}
    X^{k+1}= X^k - \eta_k \Lambda_{\xi^k, \zeta^k, \zeta'^k}(X^k) \\
    \text{with}\quad \Lambda_{\xi, \zeta, \zeta'}(X)=\relgrad_{\xi, \zeta, \zeta'}(X)+\omega\nabla\cN_{\zeta, \zeta'}(X),
    \end{gathered}\label{eq:landing_generalized_stiefel}
\end{equation}
in $\bR^{n\times p}$ whose two components are
\begin{equation*}%\label{eq:landing_components_generalized_stiefel}
    \begin{aligned}
        \relgrad_{\xi, \zeta, \zeta'}(X) &= 2\,\sk\left( \nabla f_\xi(X) X^\top B_\zeta\right) B_{\zeta'}X, \\
    \nabla \cN_{\zeta, \zeta'}(X) &= 2B_{\zeta'} X \left( X^\top B_{\zeta} X - I_p \right),
    \end{aligned}
\end{equation*}
where $\omega>0$, $\nabla \cN_{\zeta, \zeta'}(X)$ is an unbiased stochastic estimator of the gradient of $\cN(X) = \frac12 \| X^\top B X - I_p\|_{\mathrm{F}}^2$, and $\sk(A) = (A-A^\top)/2$. 
The above landing field formula~\eqref{eq:landing_generalized_stiefel} applies in the general case when both the function $f$ and the constraint matrix $B$ are stochastic; the deterministic case is recovered by substituting $\nabla f_\xi = \nabla f$ and $B_\zeta =B_{\zeta'} = B$. Note that in many applications of interest, $B_\zeta = \sum_{i=1}^r x_ix_i^\top / r$ is a subsampled covariance matrix with batch-size $r$, that is of rank at most $r$ when $r\leq n$. %While performing a projection on a subsampled matrix $B_\zeta$ is problematic, due to it being rank-deficient and requiring to either accumulate more samples or add an artificial regularization $B_\zeta+\alpha I$, 
Unlike retractions, the landing method benefits in this setting since the cost of multiplication by $B_\zeta$, which is the dominant cost of \eqref{eq:landing_generalized_stiefel}, becomes $\bigO(npr)$ instead of $\bigO(n^2p)$ where $r$ is the batch size. The landing method never requires to form the matrix $B$, thus having space complexity defined by only saving the iterates and the samples: $\bigO\left(n(p+r)\right)$ instead of $\bigO(n^2)$.

We prove that the landing iteration converges to $\tol$-critical points, i.e., points $X$ such that $\|\mathrm{grad}f(X)\|\leq e$ (where $\mathrm{grad}f$ denotes the Riemannian gradient defined in~\eqref{eq:riemannian_gradient}) and $\|\cN(X)\|\leq \tol$, with a fixed step-size in the deterministic case (Theorem~\ref{thm:landing_convergence}) and with a decaying step-size in the stochastic case (Theorem~\ref{thm:landing_stochastic_convergence}), with a rate that matches that of deterministic \citep{Boumal2019Global} and stochastic \Rev{\citep[Theorem 5, Sec.\ B]{Zhang2016Riemannian}} Riemannian gradient descent on $\stiefelG{B}$. The advantages of the landing field in~\eqref{eq:landing_generalized_stiefel} are that i) its  computation involves only parallelizable matrix multiplications, which is cheaper than the computations of the Riemannian gradient and retraction and ii) it handles gracefully the stochastic constraint, while Riemannian approaches need form the full estimate of  $B$. 

While the presented theory holds for a general smooth, possibly non-convex objective $f$, a particular problem that can be either phrased as \eqref{eq:optimization_generalized_stiefel} or framed as an optimization over the product manifold of two $\stiefelG{B}$ is CCA, which is a widely used technique for measuring similarity between datasets \citep{Arora2017Stochastic}. CCA aims to find the $p$-dimensional subspaces $X,Y\in\bR^{n\times p}$ on which the projections of the two zero-centered datasets $D_1=(d_1^1, \dots, d_1^N), D_2=(d_2^1, \dots, d_2^N)\in\bR^{n\times N}$ of $N$ i.i.d.\ samples are maximally correlated
\begin{equation}
    \begin{gathered}
    \min_{X,Y\in\bR^{n\times p}} \bE_{i}\left[-\Tr(X^\top d_1^i(d_2^i)^\top Y)\right] \\
    %\text{\,s.\,t.\,}
        X^\top \bE_{i} [d_1^i(d_1^i)^\top] X = I_p \text{ and }Y^\top \bE_{i}
        [d_2^i(d_2^i)^\top] Y = I_p,
    \end{gathered}\label{eq:optimization_cca}
\end{equation}
where the expectations are w.r.t.\ the uniform distribution over $\{1,\ldots, N\}$. Here, the constraint matrices $B_\zeta$ correspond to individual or mini-batch sample covariances, and the constraints are that the two matrices $X,Y\in\bR^{n\times p}$ are in the generalized Stiefel manifold. The proposed landing method is able to solve \eqref{eq:optimization_cca} while only having a stochastic estimate of the covariance matrices.

The rest of the introduction gives a brief overview of the current optimization techniques for solving \eqref{eq:optimization_generalized_stiefel} and its forthcoming generalization \eqref{eq:optimization_problem} when the feasible set is deterministic, since we are not aware of existing techniques for \eqref{eq:optimization_generalized_stiefel} with stochastic feasible set. Afterwards, the paper is organized as follows:
\begin{itemize}[leftmargin=1em]
    \item In Section~\ref{sec:landing} we give a form to a generalized landing algorithm for solving a smooth optimization problem $\min_{x\in\cM} f(x)$ on a smooth manifold $\cM$ defined below in ~\eqref{eq:optimization_problem}. Under suitable conditions, the algorithm converges to critical points with the same sublinear rate, $\mathcal{O}(1/K)$, as its Riemannian counterpart \citep{Boumal2019Global}, see Theorem~\ref{thm:landing_convergence}. Unlike in \citet{Schechtman2023Orthogonal}, our analysis is based on a smooth merit function allowing us to obtain a convergence result for the stochastic variant of the algorithm, when having an unbiased estimator for the landing field, see Theorem~\ref{thm:landing_stochastic_convergence}.
    \item In Section~\ref{sec:generalized_stiefel} we build on the general theory developed in the previous section and prove that the update rule for the generalized Stiefel manifold in~\eqref{eq:landing_generalized_stiefel} converges to critical points of \eqref{eq:optimization_generalized_stiefel}, both in the deterministic case with the rate $\bigO(1/K)$, and in expectation with the rate $\bigO(1/\sqrt{K})$ in the case when both the gradient of the objective function \emph{and} the feasible set are stochastic estimates.
    \item In Section~\ref{sec:numerics} we numerically demonstrate the efficiency of the proposed method on a deterministic example of solving a generalized eigenvalue problem,  stochastic CCA and ICA.
\end{itemize}
 
\textbf{Notation and terminology.} We denote vectors by lowercase letters $x,y,z,\ldots$, matrices with uppercase letters $X,Y,Z,\ldots$, and $I_n$ denotes the $n\times n$ identity matrix. We let $\beta_i$ denote the $i^{th}$ eigenvalue of $B$ and $\kappa_B=\beta_1/\beta_n$ the condition number of $B$. Let $\rD f(x)[v] = \lim_{t\rightarrow 0}(f(x+t v) - f(x))/t$ denote the derivative of $f$ at $x$ along $v$. We let $\| \cdot \|$ denote the $\ell_2$-norm also termed Frobenius norm for matrices, whereas $\| \cdot\|_2$ denotes the operator norm induced by $\ell_2$-norm. We denote the Frobenius inner product as $\inner{\cdot\,}{\cdot}$, with respect to which we define the adjoint of a linear operator $\mathrm{A}[v]$ denoted by $\mathrm{A}^*[w]$. We say that a function $f:\mathbb{R}^{n}\rightarrow \mathbb{R}$ is $L_f$-smooth if it is continuously differentiable and its gradient is Lipschitz continous with Lipschitz constant $L_f$, i.e.,\ $\|\nabla f(x) - \nabla f(y) \|_2 \leq L_f \| x-y \|_2$, for all $x, y\in\bR^n$.

\subsection{Prior work related to optimization on the generalized Stiefel manifold}

\paragraph{Riemannian optimization.}
A widely used approach to solving optimization problems constrained to manifolds as in \eqref{eq:optimization_problem} are the techniques from Riemannian optimization. These methods are based on the observation that smooth sets can be locally approximated by a linear subspace, which allows to extend classical Euclidean optimization methods, such as gradient descent and the stochastic gradient descent, to the Riemannian setting. For example, Riemannian gradient descent iterates $x^{k+1} = \Retr_\cM(x^k, -\eta_k \grad f(x^k))$, where $\eta_k>0$ is the step-size at iteration $k$, $\grad f(x^k)$ is the Riemannian gradient that is computed as a projection of $\nabla f(x^k)$ on the tangent space of $\cM$ at $x^k$, and $\Retr$ is the \emph{retraction} operation, which maps the updated iterate along the direction $- \eta_k \grad f(x^k)$ onto the manifold and is accurate up to the first-order, i.e.,  $\Retr_\cM(x, d) = x + d + o(\|d\|)$. Retractions allow the implementation of Riemannian counterparts to classical Euclidean methods on the generalized Stiefel manifold, such as Riemannian (stochastic) gradient descent \citep{Bonnabel2013Stochastic,Zhang2016Firstorder}, trust-region methods \citep{Absil2007TrustRegion}, and accelerated methods \citep{Ahn2020Nesterov}; for an overview, see \citet{Absil2008Optimization, Boumal2023introduction}. 

\begin{table}[b]
    \setlength{\tabcolsep}{4pt}
    \centering
    \scriptsize
    \begin{tabular}{lc c }
    	\toprule
         & Matrix factorizations &  Complexity \\\midrule
         Polar &  matrix inverse square root & $\bigO(n^2p)$ \\
         SVD-based & SVD & $\bigO(n^2p)$\\
         Cholesky-QR  & Cholesky, matrix inverse &  $\bigO(n^2p)$\\
         $\Lambda(X)$ formula in \eqref{eq:landing_generalized_stiefel} & None  & $\min\{ \bigO(n^2p), \bigO(npr)\}$  \\
         \bottomrule
         %  Polar - \citep{Yger2012Adaptive}
         % SVD-based \citep{Mishra2016Riemannian}
         % Cholesky-QR \citep{Sato2019Cholesky}
    \end{tabular}
    \caption{Cost comparison of retractions and the landing formula on the generalized Stiefel manifold. We assume naive flop count for the matrix-matrix multiplication and no additional structure on matrix $B_\zeta$ apart from being rank-$r$ in the stochastic setting. The matrices are of size $n\times p$ with $p \leq n$, and $r$ is the rank of the stochastic matrices $B_\zeta$. Matrix factorizations are hard to parallelize. The retractions do not allow for reduced complexities when $B_\zeta$ is low-rank and are not suited for stochastic $B_\zeta$. For the numerical timings, see Figure~\ref{fig:3_cost_retractions_cuda} in the appendices.}\label{tab:costs_operations}
\end{table}

There are several ways to compute a retraction to the generalized Stiefel manifold, which we summarize in Table~\ref{tab:costs_operations} and we give a more detailed explanation in \autoref{app:retractions}. Overall, we see  that the landing field~\eqref{eq:landing_generalized_stiefel} is much cheaper to compute than all these retractions in two cases: i) when $n \simeq p$, then the bottleneck in the retractions becomes the matrix factorizations, which, although they are of the same complexity as matrix multiplications, are much more expensive and hard to parallelize,
ii) when $n\gg p$, the dominant cost of all retractions lies in matrix multiplications that require in practice $\mathcal{O}(n^2p)$, whereas the use of the batches of size $r$ mentioned above allows computing the landing field in $\mathcal{O}(npr)$. We demonstrate numerically the practical cost of computing retractions in Figure~\ref{fig:3_cost_retractions_cuda} in the appendices.

\paragraph{Infeasible optimization methods.}
A popular approach for solving constrained optimization is to employ the squared $\ell_2$-penalty method by adding the $\omega \cN(X)$ regularizer to the objective. However, unlike the landing method, the iterates of the squared penalty method do not converge to the feasible set for any fixed choice of $\omega$ and converge only when $\omega$ goes to $\infty$ \citep{Nocedal2006Numerical}. In contrast, the landing method provably converges to the feasible set for any fixed $\omega>0$, which is enabled by the structure of the landing field~\eqref{eq:landing_generalized_stiefel} as the sum of two orthogonal components, the second one being the gradient of the infeasibility measure $\mathcal{N}$.

Augmented Lagrangian methods seek to solve a deterministic minimization problem with an augmented Lagrangian function $\cL(x,\lambda)$, such as the one introduced later in \eqref{eq:fletcher_lagrangian}, by updating the solution vector $x$ and the vector of Lagrange multipliers $\lambda$ respectively \citep{Bertsekas1982Constrained}. This is typically done by solving a sequence of optimization problems of $\cL(\cdot, \lambda_k)$ followed by a first-order update of the multipliers $\lambda_{k+1} = \lambda_k - 2 \beta h(x^k)$ depending on the penalty parameter $\beta$. The iterates are gradually pushed towards the feasible set by increasing the penalty parameter $\beta$. However, each optimization subproblem may be expensive, and the methods are sensitive to the choice of the penalty parameter~$\beta$.

Recently, a number of works explored the possibility of infeasible methods for optimization on Riemannian manifolds, when the feasible set is deterministic, in order to eliminate the cost of retractions, which can be limiting in some situations, e.g.,~when the evaluation of stochastic gradients of the objective is cheap. The works of \citet{Gao2019Parallelizable, Gao2022orthogonalization} proposed a modified augmented Lagrangian method which allows for fast computation and better bounds on the penalty parameter $\beta$. \citet{Ablin2022Fast} designed a simple iterative method called \emph{landing}, consisting of two orthogonal components, to be used on the orthogonal group, which was later expanded to the Stiefel manifold \citep{Gao2022Optimization,Ablin2023Infeasible}. \citet{Schechtman2023Orthogonal} expanded the \emph{landing} approach to be used on a general smooth constraint using a non-smooth merit function. More recently, \citet{Goyens2024Computing} analysed the classical Fletcher's augmented Lagrangian for solving smoothly constrained problems through the Riemannian perspective and proposed an algorithm that provably finds second-order critical points of the minimization problem. As the differentiability of the infeasible models relies on the second-order information of the objective, \citet{Xiao2023dissolving} proposed a constraint-dissolving model where the exact gradient and Hessian are convenient to compute. 
% Random constrained optimization : \citep{Wang2016Stochastic}

\subsection{Existing methods for the GEVP and CCA}
% Parallel QR factorization for tall skinny matrices: https://scicomp.stackexchange.com/questions/1026/when-do-orthogonal-transformations-outperform-gaussian-elimination/1030#1030
\begin{table*}[h]
    \setlength{\tabcolsep}{10pt}
    \centering
    \scriptsize
    \centerline{
    \begin{tabular}{lcc l l l}
    \toprule
         &  Stochastic & Matrix factorizations & Convergence & Per-iteration complexity  & Memory \\ \midrule
    AppGrad \citep[Theorem 2.1]{Ma2015Finding}   & - & SVD & local linear & $\bigO(n^2p + p^3)$ & $n^2$  \\
    \texttt{CCALin} \citep[Theorem 7]{Ge2016Efficient}     & - & inexact linear solver & global linear & $\bigO(n^2p + p^3)$ & $n^2$ \\
    \texttt{rgCCALin} \citep[Theorem 6.1]{Xu2020Practical}     & - & inexact linear solver & global linear & $\bigO(n^2p + p^3)$ & $n^2$ \\
    %ALS  \citep{Wang2016Efficient}     & - & linear solver & $\bigO( )$\\ This is method for the top-1 case
    LazyCCA \citep[Theorem 4.2]{Allen-Zhu2017Doubly}      & - & inexact linear solver & global linear & $\bigO( n^2p + p^2 n)$ & $n^2$\\
    %Gen-Oja \citep{Bhatia2018GenOja}  & - & matrix mult. & $\bigO(n^2p)$\\ %This one is used only for top-1 vector
    %\citep{Gao2019Stochastic}      & \checkmark  & var.~red. linear solve & $\bigO(n^3)$\\
    %\texttt{RSG+}\citep{Meng2021Online}    & - & matrix exponential & $\mathcal{\tilde O}(n^2p/\varepsilon)$ & \\
    MSG \citep[Theorem 2.3]{Arora2017Stochastic} & \checkmark & inverse square root &  global sublinear & $\bigO(n^3)$ & $n^2$ \\
    $\Lambda(X)$ formula in \eqref{eq:landing_generalized_stiefel} & \checkmark & None & global sublinear & $\bigO(npr)$ & $n (p+r)$\\\bottomrule
    %\hline
    \end{tabular}}
    \caption{Summary of CCA and GEVP solvers for finding top-$p$ vectors simultaneously. For CCA based on covariance matrices we assume that the number of samples is much greater than the dimension,~\emph{i.e.},~$N\gg n$.  ``Stochastic'' marks methods with convergence analysis in expectation for the stochastic case. We assume that deterministic methods require forming the matrix $B$ at the start with additional cost $\bigO( Nn^2)$ and store it in iterations to remove dependence of the complexity on $N$.}
    % For time complexities, a good overview can be found here:https://proceedings.neurips.cc/paper_files/paper/2016/file/42998cf32d552343bc8e460416382dca-Paper.pdf
    \label{tab:literature}
\end{table*}

\paragraph{Deterministic methods.}
A lot of effort has been spent in recent years on finding fast and memory-efficient solvers for CCA and the GEVP. The top-$p$ GEVP, that seeks to find the eigenspace corresponding to the $p$ largest eigenvalues of the pair $(A,B)$, can be formulated as~\eqref{eq:optimization_generalized_stiefel}; this can be deduced from~\citet[Proposition~2.2.1]{Absil2008Optimization}. As for CCA, it can be framed as \eqref{eq:optimization_generalized_stiefel} \citep{Ge2016Efficient,Bhatia2018GenOja} or as a minimization over a Cartesian product of two generalized Stiefel manifolds as in \eqref{eq:optimization_cca}. The majority of the existing methods specialized for CCA and the GEVP that compute the top-$p$ vector solution aim to circumvent the need to compute $B^{-\frac12}$ or $B^{-1}$, e.g.,~by using an approximate solver to compute the action of multiplying with $B^{-1}$. The classic Lanczos algorithm for computation of eigenvalues can be adapted to the GEVP by noting that we can look for standard eigenvectors of $B^{-1}A$, see \citep[Algorithm 9.1]{Saad2011Numerical}. 
\citet{Ma2015Finding} propose \texttt{AppGrad} which performs a projected gradient descent with $\ell_2$-regularization and proves its convergence when initialized sufficiently close to the minimum. %\cite{Wang2016Efficient} designs an alternating least square method for the top-$1$ vector of CCA that is employed on transformed coordinates by computing the inverse square root of the covariance matrix only once and then employs only linear solvers that are efficiently employed via ridge-regression subproblem. 
The work of \citet{Ge2016Efficient} proposes \texttt{GenELinK} algorithm based on the block power method, using inexact linear solvers, that has provable convergence with a rate depending on $1/\delta$, where $\delta=\beta_p - \beta_{p+1}$ is the eigenvalue gap. %\citep{Wang2016Efficient} proposes an alternating least squares scheme that also has mini-batch variant. 
\citet{Allen-Zhu2017Doubly} improve upon this in terms of the eigenvalue gap and proposes the doubly accelerated method \texttt{LazyCCA}, which is based on the shift-and-invert strategy with iteration complexity that depends on $1/\sqrt{\delta}$. \citet{Xu2020Practical} present a first-order Riemannian algorithm that computes gradients using fast linear solvers to approximate the action of $B^{-1}$ and performs polar retraction. 
%\citep{Meng2021Online} presents a Riemannian optimization technique that finds top-$p$ vectors using online estimates of the covariance matrices with $\bigO(n^2p)$ per-iteration computational cost and convergence rate of $\bigO(1/K)$  when the objective is geodesically convex.

\paragraph{Stochastic methods.} 
While the stochastic CCA problem is of high practical interest, fewer works consider it. Although several of the aforementioned deterministic solvers can be implemented for streaming data using sampled information \citep{Ma2015Finding,Wang2016Efficient,Meng2021Online}, they do not analyse stochastic convergence. The main challenge comes from designing an unbiased estimator for the whitening part of the method that ensures the constraint $X^\top B X = I$ in expectation. \citet{Arora2017Stochastic} propose a stochastic approximation algorithm, MSG, that keeps a running weighted average of covariance matrices used for projection, requiring computing $B^{-1/2}$ at each iteration. Additionally, the work of \citet{Gao2019Stochastic} proves stochastic convergence of an algorithm based on the shift-and-invert scheme and SVRG to solve linear subproblems, but only for the top-$1$ setting.

\paragraph{Comparison with the landing.} Constrained optimization methods such as the augmented Lagrangian methods and Riemannian optimization techniques can be applied on stochastic problems when the gradient of the objective function is random but not on problems when the \emph{feasible set is random}. The landing method has provable global convergence guarantees with the same asymptotic rate as its Riemannian counterpart, while also allowing for stochasticity in the constraint. Our work is conceptually related to the recently developed infeasible methods \citep{Ablin2022Fast, Ablin2023Infeasible, Schechtman2023Orthogonal}, with the key difference of constructing a smooth merit function for a general constraint $h(x)=0$, which is necessary for the convergence analysis of stochastic iterative updates that can have error in the normal space of $\cM$. In Table~\ref{tab:literature} we show the overview of relevant GEVP/CCA methods by comparing their per-iteration complexity, memory requirements, and the type of proved convergence. %asymptotic operations cost required to converge to $\tol$-critical points. %\footnote{Note that some of the works show \emph{linear} convergence, \emph{i.e.} $\log(1/\tol)$, to a global minimizer, which by the smoothness of $f$ also implies a $\tol$-critical point, whereas we prove $1/\tol^2$ convergence to critical points. For the purpose of the comparison, we overlook this difference. Also, there are no local non-global minimizers in the GEVP.}    
Despite the landing iteration \eqref{eq:landing_generalized_stiefel} being designed for a general non-convex smooth problem~\eqref{eq:optimization_generalized_stiefel} and not being tailored specifically to the GEVP/CCA, we achieve theoretically interesting rate of convergence. Additionally, we provide an improved space complexity $\bigO\left(n(p+r)\right)$ by not having to form the full matrix $B$ and only to save the iterates and the streaming samples.

\section{Landing on General Stochastic Constraints\label{sec:landing}}
This section is devoted to analyzing the landing method in the general case where the feasible set is given by the zero set of a smooth function. We will later use these results in Section~\ref{sec:generalized_stiefel} devoted to extending and analyzing the landing method~\eqref{eq:landing_generalized_stiefel} on $\stiefelG{B}$. The theory presented here improves on that of \citet{Schechtman2023Orthogonal} in two important directions. First, we introduce the notion of relative ascent direction, which allows us to consider a richer class than that of \emph{geometry-aware orthogonal directions}~\citep[Eq.\ 18]{Schechtman2023Orthogonal}. Second, we do not require any structure on the noise term $\tilde E$ defined later in \eqref{eq:landing_random_general}, for the stochastic case, while A2(iii) in \citet{Schechtman2023Orthogonal} requires the noise to be in the tangent space. This enhancement is due to the smoothness of our merit function $\mathcal{L}$, while \citet{Schechtman2023Orthogonal} consider a non-smooth merit function. Importantly, for the case of $\stiefelG{B}$ with the formula given in~\eqref{eq:landing_generalized_stiefel}, there is indeed noise in the normal space, rendering \citet{Schechtman2023Orthogonal}'s theory inapplicable, while we show in the next section that \autoref{thm:landing_stochastic_convergence} applies in that case.

% Problem statement
Given a continuously differentiable function $f:\bR^{d}\rightarrow \bR$, we address the optimization problem
\begin{equation}
    \min_{x\in\bR^d} f(x) \quad \mathrm{s.\,t.}\quad x\in\cM = \left\{x\in\bR^d\,:\,h(x)=0 \right\}, \label{eq:optimization_problem}
\end{equation}
where $h:\bR^{d}\rightarrow\bR^q$ is continuously differentiable, $q\in\mathbb{N}$ represents the number of constraints, and $\cM$ defines a smooth manifold set.
% eps set
We will consider algorithms that stay within an initially prescribed \emph{$\varepsilon$-safe region} 
\begin{equation*}
    \cM ^\varepsilon = \left\{x\in\bR^d \,:\, \| h(x) \|\leq \varepsilon \right\},
\end{equation*}
which can be split into a collection of layered manifolds \citep{Goyens2024Computing}
\begin{equation*}
    \cM_c = \left\{x\in\bR^d \,:\, h(x)=c \right\},
\end{equation*}
with $\| c \| \leq \varepsilon$.

% Assumptions
The first assumption we make is that the gradient of $f$ is Lipschitz continuous. The second one requires that the differential $\mathrm{D}h(x)$ inside the $\varepsilon$-safe region has bounded singular values.
\begin{assumption}[Smoothness of the objective]\label{ass:fgrad_lipschitz}
    The objective function $f:\mathbb{R}^d\to\mathbb{R}$ is $L_f$-smooth.
\end{assumption}
\begin{assumption}[Smoothness of the constraint]\label{ass:h_constraint}
    %Let $\rD h (x)^*:\bR^q\rightarrow \bR^d$ be the adjoint of the differential of the constraint function $h$. 
    The differential of the constraint function has bounded singular values for $x$ in the $\varepsilon$-safe region, i.e.,
    \begin{equation*}
        \forall x\in\cM^\varepsilon: \quad \constHlow \leq \sigma\left( \rD h(x) \right) \leq \constH.
    \end{equation*}
    Additionally, the penalty $\cN(x) = \frac12 \| h(x)\|^2$ is $L_{\mathcal{N}}$-smooth over $\mathcal{M}^\varepsilon$.
\end{assumption}
Assumption~\ref{ass:fgrad_lipschitz} is standard in optimization. Assumption~\ref{ass:h_constraint} is necessary for the analysis of smooth constrained optimization \citep{Goyens2024Computing} and holds, e.g., when $\cM^\varepsilon$ is a compact set, and $\rD h(x)$ has full rank for all $x\in\cM^\varepsilon$. This ensures that every layered manifold $\cM_c$ is an embedded submanifold of $\mathbb{R}^d$. The tangent space to $\mathcal{M}_c$ at $x$ is the null space of $\mathrm{D}h(x)$, the normal space at $x$ (in the sense of the Frobenius inner product) is the range (i.e., image) of $\mathrm{D}h(x)^*$, and a \emph{critical point} is then a point $x$ in $\mathcal{M}$ where $\nabla f(x)$ belongs to the normal space.

% Relative ascent
Next we define the notion of relative ascent direction, used to guarantee that~\eqref{eq:landing_generalized_stiefel} produces a descent when the second term of the landing field vanishes.
\begin{definition}[Relative ascent direction\label{def:relative_descent_direction}]
    A \emph{relative ascent direction} $\relgrad(x):\bR^d\rightarrow\bR^d$, with a parameter $\rho>0$ that may depend on $\varepsilon$, satisfies:
	\begin{enumerate}[label=(\roman*)]
		\item{$\forall x\in\cM^\varepsilon, \quad\forall v \in \mathrm{range}(\rD h(x)^*): \inner{\relgrad(x)}{v} = 0$;}
		\item{$\forall x\in\cM^\varepsilon, \quad \inner{\relgrad(x)}{\nabla f(x)} \geq \rho \| \Psi(x) \|^2$;}
		\item{$\forall x\in\cM, \quad \inner{\relgrad(x)}{\nabla f(x)} = 0$ if and only if $x$ is a critical point of $f$ subject to $\cM$.}
	\end{enumerate}
\end{definition}
In short, the relative ascent direction must be in the tangent space to every layered manifold $\cM_{h(x)}$ while remaining positively aligned with the Euclidean gradient $\nabla f(x)$. Note that the above definition is not scale invariant to $\rho$, i.e.,~taking $c\relgrad(x)$ for $c>0$ will result in $c\rho$, and this is in line with the forthcoming convergence guarantees deriving an upper bound on $\|\relgrad(x)\|$. While there may be many examples of relative ascent directions, a particular example is given next.
\begin{definition}[Riemannian gradient on the layered manifold $\cM_c$]\label{def:riemannian_gradient}
    Let $f:\cM_c\rightarrow \bR$ be a smooth function on $\cM_c$. The \emph{Riemannian gradient} of $f$, denoted by $\grad f$, is uniquely defined by
    \begin{equation*}
        \forall x\in\cM_c, v\in\mathrm{T}_x\cM_c, \quad \rD f(x)[v] = \inner{v}{\grad f(x)},
    \end{equation*}
    where $\mathrm{T}_x\cM_c$ denotes the tangent space of $\cM_c$ at $x$.
\end{definition}

\begin{proposition}[Riemannian gradient is a relative ascent direction]\label{prop:riemannian_gradient}
    The Riemannian gradient defined in Definition~\ref{def:riemannian_gradient} is a relative ascent direction on $\cM^\varepsilon$ with $\rho = 1$.
\end{proposition}
The proof can be found in the appendices in~\autoref{proof:riemannian_gradient}. Such extension of the Riemannian gradient to the collection of layered manifolds was already considered by~\citet{Gao2022Optimization} in the particular case of the Stiefel manifold and by~\citet{Schechtman2023Orthogonal}.

\subsection{Deterministic case\label{subsec:deterministic_case}}
We now define the general form of the \emph{deterministic} landing iteration as
\begin{equation}
    x^{k+1} = x^k - \eta_k \Lambda(x^k)\, \text{with}\, \Lambda(x) = \relgrad(x) + \omega \nabla \cN(x), \label{eq:landing_iteration}
\end{equation}
where $\relgrad(x)$ is a relative ascent direction described in  Definition~\ref{def:relative_descent_direction}, $\nabla\cN(x)=\mathrm{D}h(x)^* h(x)$ is the gradient of the penalty $\cN(x) = \frac12 \| h(x)\|^2$, $\omega>0$ is a parameter, and $\|\cdot\|$ is the $\ell_2$-norm. %The stochastic iterations, where noise is added at each iteration, will be introduced later in \eqref{eq:landing_random_general}. 
Condition (i) in Definition~\ref{def:relative_descent_direction} guarantees that  $\inner{\nabla \cN(x)}{\relgrad(x)}=0$, so that the two terms in $\Lambda$ are orthogonal. 

Note that we can use \emph{any} relative ascent direction as $\relgrad(x)$. The Riemannian gradient in \eqref{eq:riemannian_gradient} is just one special case, which has some shortcomings. Firstly, it requires a potentially expensive projection $\rD h(x)^* \left(\rD h(x)^*\right)^\dagger$. Secondly, if the constraint involves a random noise on $h$, formula~\eqref{eq:riemannian_gradient} does not give an unbiased formula in expectation. An important contribution of the present work is the derivation of a computationally convenient relative ascent direction in the specific case of the generalized Stiefel manifold in Section~\ref{sec:generalized_stiefel}.

We now turn to the analysis of the convergence of this method.
The main object allowing for the convergence analysis is Fletcher's augmented Lagrangian
\begin{equation}
    \cL(x) = f(x) - \inner{h(x)}{\lagrM(x)} + \beta \| h(x) \|^2, \label{eq:fletcher_lagrangian}
\end{equation}
with the Lagrange multiplier $\lambda(x)\in\bR^{p}$ defined as $\lagrM(x) = (\rD h(x)^*)^\dagger [\nabla f(x)]$ \citep{Goyens2024Computing}. The next assumption that the differential of $\lambda(x)$ is bounded is met when $\cM^\varepsilon$ is a compact set.
\begin{assumption}[Multipliers of Fletcher's augmented Lagrangian]\label{ass:lagrange_mult}
    The norm of the differential of the multipliers of Fletcher's augmented Lagrangian is bounded: $
        \sup_{x\in\cM^\varepsilon}\|\rD \lambda(x) \|\leq \constLagrM.$
\end{assumption}
\begin{proposition}[Lipschitz constant of Fletcher's augmented Lagrangian]
    Fletcher's augmented Lagrangian $\cL$ in \eqref{eq:fletcher_lagrangian} is $L_\cL$-smooth on $\cM^\varepsilon$, with $L_\cL = L_{f+\lambda} + 2\beta L_\cN$, where  $L_{f+\lambda}$ is the smoothness constant of $f(x)-\inner{h(x)}{\lambda(x)}$ and $L_\cN$ is that of $\cN(x)$.
\end{proposition}
\begin{proof}
    By the smoothness of $f(x)-\inner{h(x)}{\lambda(x)}$ and $\cN(x)$ combined with the triangle inequality for $\|\cdot\|$.
\end{proof}

The following two lemmas show that there exists a positive step-size $\eta$ that guarantees that the next landing iteration stays within $\cM^\varepsilon$ provided that the current iterate is inside $\cM^\varepsilon$.
\begin{lemma}[A step-size safeguard] \label{lemma:safe_step}
Let $x\in\cM^\varepsilon$ and consider the iterative update $\tilde x = x - \eta \Lambda(x)$, where $\eta >0$ is a step-size and $\Lambda(x)$ is the landing field in \eqref{eq:landing_iteration}. If the step-size satisfies $\eta \leq \eta(x)$ with
\begin{align*}
    \eta(x):= &~\frac{1}{{L_\cN \|\Lambda(x)\|^2}}\Big( \omega \| \nabla \cN(x) \|^2 + \\
        &\sqrt{\omega^2 \|\nabla  \cN(x)\|^4 + L_\cN \|\Lambda(x)\|^2 (\varepsilon^2 -\| h(x)\|^2)} \Big)
\end{align*}
where $L_\cN$ is from Assumption~\ref{ass:h_constraint}, then the line segment from the current to the next iterate remains in the safe region.
\end{lemma}
The proof can be found in the appendices in~\autoref{proof:safe_step}. Next, we require that the norm of the relative ascent direction must remain bounded in the safe region.
\begin{assumption}[Bounded relative ascent direction]\label{ass:rel_descent_bound}
    We require that $
        \sup_{x\in\cM^\varepsilon} \| \relgrad(x) \| \leq \constRgrad.$
\end{assumption}
This holds, for instance, if $\nabla f$ is bounded in $\cM^\varepsilon$, using Definition~\ref{def:relative_descent_direction} (ii) and Cauchy-Schwarz inequality. Under this assumption, we can lower bound the step-size safeguard in Lemma~\ref{lemma:safe_step} for all $x\in\cM^\varepsilon$, implying that there is always a positive step-size that keeps the next iterate in the safe region.
\begin{lemma}[A lower-bound on the step-size safeguard]\label{lemma:safe_step_lower}
    The step-size safeguard $\eta(x)$ in Lemma~\ref{lemma:safe_step} is lower bounded away from zero by
    \begin{equation*}
        \begin{split}
            \underline \eta := \min\Bigg\{  &\frac{\omega \constHlow^2 \alpha^2 \varepsilon^2}{ L_\cN \left( \constRgrad^2 + \omega^2 \constH^2 \varepsilon^2 \right)}, \frac{(1-\alpha) \varepsilon}{\sqrt{2L_\cN}}, \\
            &\frac{(1-\alpha)\varepsilon}{\sqrt{2L_\cN}\left(  \constRgrad^2 + \omega^2 \constH^2 \varepsilon^2 \right)}, 
            \frac{1}{\omega L_\cN}\left( \frac{\constHlow}{\constH} \right)^2\Bigg\}
        \end{split}
    \end{equation*}
    for any choice of $0<\alpha<1$ where $\constH, \constHlow, \constRgrad>0$ are constants from Assumption~\ref{ass:h_constraint} and \ref{ass:rel_descent_bound}.
\end{lemma}
The proof can be found in~\autoref{proof:safe_step_lower}.

\begin{lemma}\label{lemma:fletcher_inner}
Let $\cL(x)$ be Fletcher's augmented Lagrangian in \eqref{eq:fletcher_lagrangian} with $\beta \geq(\frac{\rho}{4\constH^2} + \frac{\omega\constLagrM}{2\constH} + \frac{\constLagrM^2}{4 \rho\constH^2})/\omega$, where $\rho$ is defined in Definition~\ref{def:relative_descent_direction}. We have that $
    \inner{\nabla \cL(x)}{ \Lambda(x)} \geq  \frac\rho2 \left( \| \relgrad(x)\|^2 + \| h(x)\|^2 \right).$
\end{lemma}
The proof can be found in the appendices in~\autoref{proof:fletcher_inner}. This critical lemma shows that $\mathcal{L}$ is a valid merit function for the landing iterations and allows the study of convergence of the method with ease.

The following statement combines Lemma~\ref{lemma:fletcher_inner} with the bound on the step-size safeguard in Lemma~\ref{lemma:safe_step_lower} to prove sublinear convergence to critical points of $f$ subject to $\cM$.
\begin{theorem}[Convergence of the deterministic landing]\label{thm:landing_convergence}
    Under the above assumptions and with constant step-size $\eta$, the landing iteration in \eqref{eq:landing_iteration} starting from $x_0\in\cM^\varepsilon$ satisfies
    \begin{align*}
        \frac1K \sum_{k=0}^K \| \relgrad(x^k) \|^2 &\leq  4\frac{\cL(x^0) - \cL^*}{\eta \rho K},\\% \quad \text{and}\\
        \frac1K\sum_{k=0}^K \| h(x^k)\|^2  &\leq 4\frac{\cL(x^0) - \cL^*}{\eta \rho K}.
    \end{align*}
    for $\eta \leq \min\left\{ \frac{\rho}{2L_\cL}, \frac{\rho}{2L_\cL \Rev{\omega^2}\constH^2}, \underline \eta \right\}$, where $\underline \eta$ comes from Lemma~\ref{lemma:safe_step_lower}, and  $\cL^* = \min_{x\in\cM^\varepsilon} \cL(x)$.
\end{theorem}
The proof is given in~\autoref{proof:landing_convergence} and implies that the iterates $x^k$ converge to critical points with the sublinear rate $\bigO(1/K)$.

\subsection{Stochastic case\label{subsec:stochastic_case}}
Due to the smoothness of Fletcher's augmented Lagrangian in the $\cM^\varepsilon$ region, we can extend the convergence result to the stochastic setting. The iteration is
\begin{equation}
\label{eq:landing_random_general}
    x^{k+1}=x^k - \eta_k \left[\Lambda(x^k) + \tilde E(x^k, \Xi^k)\right],
\end{equation}
where the $\Xi^k$ are i.i.d.\ random variables, $\tilde E(x^k, \Xi^k)$ is the random error term at iteration $x^k$, and $\Lambda(x^k)$ is the landing field in \eqref{eq:landing_iteration}. We require the landing update in \eqref{eq:landing_random_general} to be an unbiased estimator with bounded variance.
\begin{assumption}[An unbiased estimator of $\Lambda(x^k)$ with bounded variance]\label{ass:noise}
    There exists $\var >0$ such that for all $x\in\cM^\varepsilon$, we have $\bE_\Xi[\tilde E(x, \Xi)] = 0$ and $\bE_\Xi[ \| \tilde E(x, \Xi)\|^2 ] \leq \gamma^2$.
\end{assumption}
We obtain the following result with decaying step-sizes.
\begin{theorem}[Convergence of the stochastic landing]\label{thm:landing_stochastic_convergence}
    Under the above assumptions, the stochastic landing iteration in \eqref{eq:landing_random_general} with a diminishing step-size $\eta_k = \eta_0 \times (1+k)^{-1/2}$, and assuming the line segments between the iterates remain within $\cM^\varepsilon$ with probability one, produces iterates for which
    \begin{align*}
        \inf_{k\leq K} \bE\left[ \| \relgrad(x^k) \|^2 \right] &\leq 4\frac{\cL(x^0) - \cL^*}{\rho \eta_0 \sqrt{K}} \\&\qquad+ \frac{ 2 \eta_0 \gamma^2 L_\cL (1 + \log(K+1))}{\rho\sqrt{K}},\\%\left( \cL(x^0) - \cL^*  + c \gamma^2 \log(K) \right)\\%\qquad\text{and}\qquad 
        \inf_{k\leq K} \bE\left[ \| h(x) \|^2 \right] &\leq 4\frac{\cL(x^0) - \cL^*}{\rho \eta_0 \sqrt{K}} \\&\qquad+ \frac{ 2 \eta_0 \gamma^2 L_\cL (1 + \log(K+1))}{\rho\sqrt{K}},
    \end{align*}
    for $\eta_0 \leq  \frac{\rho}{2L_\cL} \min\left\{ 1, (\Rev{\omega}\constH)^{-2} \right\}$ and  $\cL^* = \min_{x\in\cM^\varepsilon} \cL(x)$.
\end{theorem}
The theorem is proved in \autoref{proof:landing_stochastic_convergence}. Unlike in the deterministic case in Lemma~\ref{lemma:safe_step}, without further assumption on the distribution of $\Xi^k$, it cannot be ensured that the line segments connecting the successive iterates are within $\cM^\varepsilon$ with probability one. Under that assumption, we recover the same convergence rate as Riemannian SGD in the non-convex setting for a deterministic feasible set \Rev{\citep[Theorem 5, Sec.\ B]{Zhang2016Riemannian}}, but in our case, we require only an online estimate of the random manifold feasible set.

\section{Landing on the Generalized Stiefel Manifold\label{sec:generalized_stiefel}}

This section builds on the results of the previous Section~\ref{sec:landing} and proves that the simple landing update rule $X^{k+1}= X^k - \eta_k \Lambda(X^k)$, defined in \eqref{eq:landing_generalized_stiefel}, converges to the critical points of \eqref{eq:optimization_generalized_stiefel}. The generalized Stiefel manifold $\stiefelG{B}$ is defined by the constraint function $h(X) = X^\top B X - I_p$, and we have $\nabla \mathcal{N}(X) = 2BX(X^TBX -I_p)$. \Rev{The specific forms of $\rD h(X)$ and $\lambda(X)$ can be found in  \autoref{app:specific_forms}.} We now derive the quantities required for Assumption~\ref{ass:h_constraint}. Recall that $\beta_i$ denotes the $i^{\mathrm{th}}$ eigenvalue of $B$ and $\kappa_B = \beta_1/\beta_n$ is the condition number of $B$.
\begin{proposition}[Smoothness constants for the generalized Stiefel manifold]\label{prop:gen_stiefel_constants}
    Smoothness constants in Assumption~\ref{ass:h_constraint} for the generalized Stiefel manifold are
    \begin{align*}
        \constH &= 2\sqrt{(1+\varepsilon)\beta_1\kappa_B}\\
        \constHlow&=2\sqrt{(1-\varepsilon)\beta_n\kappa_B^{-1}}\\
        L_\cN &= 2 \beta_1\left(\varepsilon + 2(1+\varepsilon)\kappa_B\right).
    \end{align*}
\end{proposition}
The proof deriving $\constH, \constHlow$ is presented in~\autoref{proof:gen_stiefel_constants} and smoothness constant $L_\cN$ comes from \cref{lemma:gevp_lipschitz}.

We show two candidates for the relative ascent direction:
\begin{proposition}[Relative ascent directions for the generalized Stiefel manifold]\label{prop:generalized_stiefel_relative_descent}
    The following two formulas are relative ascent directions on the generalized Stiefel manifold:
    \begin{align*}
        \relgrad_B(X) &= 2\sk (\nabla f(X)X^\top B) B X\\
        \relgrad^{\mathrm{R}}_B(X) &= 2\sk (B^{-1} \nabla f(X)X^\top) B X
    \end{align*}
    with $\relgrad_B(X)$ having $\rho_B = 1/(\beta_1 \kappa_B (1+\varepsilon))$ and $\relgrad^\mathrm{R}_B(X)$ having $\rho^\mathrm{R}_B =\beta_n/(1+\varepsilon)$.
\end{proposition}
The proof is given in~\autoref{proof:generalized_stiefel_relative_descent}. The formula for the relative ascent $\relgrad_B^\mathrm{R}(X)$ can be derived as a Riemannian gradient for $\stiefelG{B}$ in a metric derived from a canonical metric on the standard Stiefel manifold via a specific isometry; see~\autoref{app:canonical_reldescent}.

\subsection{Deterministic generalized Stiefel case\label{subsec:deterministic_gstiefel_case}}
The fact that $\relgrad_B(X)$ above meets the conditions of Definition~\ref{def:relative_descent_direction} allows us to define the deterministic landing iterations as $X^{k+1} = X^k - \eta^k \Lambda(X^k)$ with 
\begin{multline}\label{eq:deterministic_landing_stiefel}
    \Lambda(X) = 2\,\sk (\nabla f(X)X^\top B) B X \\+ 2\omega BX(X^TBX -I_p),
\end{multline}
and \autoref{thm:landing_convergence} applies to these iterations, showing that they converge to critical points.

\subsection{Stochastic generalized Stiefel case}
One of the main features of the formulation in \eqref{eq:deterministic_landing_stiefel} is that it seamlessly extends to the stochastic case when both the objective $f$ and the constraint matrix $B$ are expectations.
Indeed, using the stochastic estimate $\Lambda_{\xi, \zeta, \zeta'}$ defined in \eqref{eq:landing_generalized_stiefel}, we have $\mathbb{E}_{\xi, \zeta, \zeta'}[\Lambda_{\xi, \zeta, \zeta'}(X)] = \Lambda(X)$. 
The stochastic landing iterations are, therefore, of the same form as in \eqref{eq:landing_random_general}. 
To apply \autoref{thm:landing_stochastic_convergence} we need to bound the variance of $\tilde{E}(X, \Xi) = \Lambda_{\xi, \zeta, \zeta'}(X) - \Lambda(X)$ where the random variable $\Xi$ is the triplet $(\xi, \zeta, \zeta')$ using standard U-statistics techniques~\citep{van2000asymptotic}.

\begin{proposition}[Variance estimation of the generalized Stiefel landing iteration]
\label{prop:variance_landing}
    Let $\sigma_B^2$ be the variance of $B_{\zeta}$ and $\sigma_G^2$ the variance of $\nabla f_\xi(X)$. We have that
    \begin{equation}
        \begin{gathered}
            \mathbb{E}_\Xi[\|\tilde{E}(X, \Xi)\|^2] \leq \sigma_G^2 \alpha_G + \sigma_B^2(\alpha_B + \omega^2 \gamma_B)
        \end{gathered}\label{eq:variance_landing}
    \end{equation}
    where the constants $\alpha_G, \alpha_B, \gamma_B$ are given explicitly in \autoref{proof:variance_landing}, and depend only on $\varepsilon$, the distribution of $B_\zeta$, and the function $f$.
    % This work \url{http://proceedings.mlr.press/v97/qian19b/qian19b.pdf} only controls the variance at the optimum. This would make many constants in our variance estimation much better. However, I'm afraid this does not work in non-convex optim
\end{proposition}
The proof is found in \autoref{proof:variance_landing}. Note that, as expected, the variance in \eqref{eq:variance_landing} is zero in the deterministic setting where both variances $\sigma_B$ and $\sigma_G$ are zero. A consequence of Proposition~\ref{prop:variance_landing} is that \autoref{thm:landing_stochastic_convergence} applies in the case of the stochastic landing method on the generalized Stiefel manifold, and more specifically, also for solving the stochastic GEVP.

\section{Numerical Experiments\label{sec:numerics}}
\paragraph{Generalized eigenvalue problem.}
We compare the methods on the deterministic top-$p$ GEVP that consists of solving $\min_{X\in\bR^{n\times p}} -\frac12 \Tr(X^\top A X)$ for $X\in\stiefelG{B}$. The two matrices are randomly generated with a condition number $\kappa_A = \kappa_B = 100$ and with the size $n = 1000$ and $p=500$; see further specifics in~\autoref{app:experiments}.\footnote{The code is available at: \url{https://github.com/simonvary/landing-generalized-stiefel}.}

Figure~\ref{fig:1_gevp} shows the timings of four methods with fixed step-size: Riemannian steepest descent with QR-based Cholesky retraction \citep{Sato2019Cholesky}, the PLAM method \citep{Gao2022orthogonalization}, and the two landing methods with either $\relgrad^\mathrm{R}_B(X)$ or $\relgrad_B(X)$ in Proposition~\ref{prop:generalized_stiefel_relative_descent}. The landing method with $\relgrad_B(X)$ converges the fastest in terms of time, due to its cheap per-iteration computation, which is also demonstrated in Figure~\ref{fig:1_gevp_iter} and Figure~\ref{fig:3_cost} in the appendices. It can be also observed that the landing method with $\relgrad_B(X)$ is more robust to the choice of parameters $\eta$ and $\omega$ compared to PLAM, which we show in Figure~\ref{fig:4_gevp_sensitivity} and Figure~\ref{fig:5_landing_sensitivity} in the appendices, and is in line with the equivalent observations previously made for the orthogonal manifold \citep[Figure~9]{Ablin2022Fast}. In Figure~\ref{fig:1_gevp_safestep} in the appendices we track numerically the value of the step-size safeguard $\eta(X)$ in Lemma~\ref{lemma:safe_step}.
\begin{figure}[t]
    \begin{center}
    \includegraphics[width=0.4\textwidth]{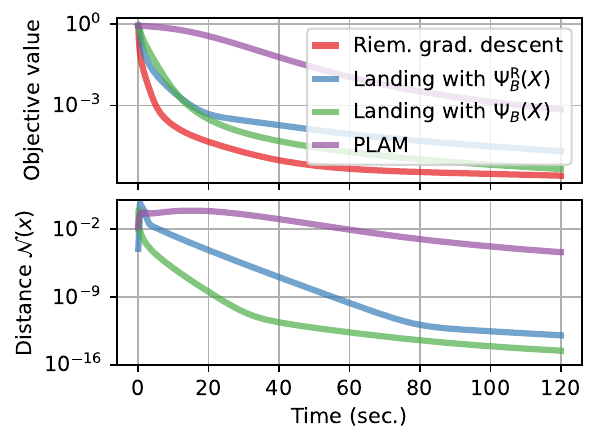}
    \end{center}
    \vspace{-1.5em}
    \caption{Generalized eigenvalue problem ($n = 1000, p = 500$).\label{fig:1_gevp}}
  %\vspace{-2em}
  \begin{center}
    \includegraphics[width=0.4\textwidth]{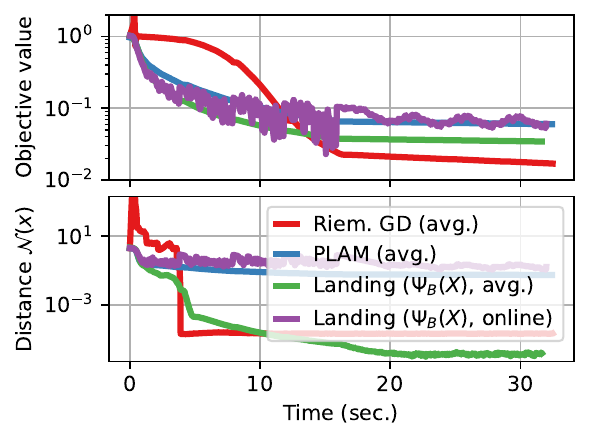}
  \end{center}
  \vspace{-1.5em}
  \caption{Stochastic CCA on the split MNIST dataset for $p=5$. An epoch takes roughly $2.5$~sec. \label{fig:2_cca_mnist}}
\end{figure}

\paragraph{Stochastic CCA and ICA.}

For stochastic CCA, we use the benchmark problem used by \citet{Ma2015Finding,Wang2016Efficient}, in which the MNIST dataset is split in half by taking left and right halves of each image, and compute the top-$p$ canonical correlation components by solving \eqref{eq:optimization_cca}. In our experiments, we have $N = 60\,000$, $n = 392 $, $p=5$, and $r = 512$.

The stochastic ICA is performed by solving~\cite{hyvarinen1999fast,ablin2018faster}
%\vspace{-.5em}
\begin{equation*}
    \min_{X\in\bR^{n\times n}} \frac1N \sum_{i=1}^N \sum_{j=1}^n \sigma([AX]_{i,j}),\, \mathrm{s.\,t.}\, X\in\mathrm{St}_{\frac1N A^\top A}(n,n)
\end{equation*}
where $\sigma(x) = \log(\mathrm{cosh}(x))$ is performed elementwise and $\sigma'(x) = \mathrm{tanh}(x)$. We generate the data matrix $A$ as $A = SW^\top$, where $S$ is a $N\times n$ matrix of random i.i.d.\ data sampled from a Laplace distribution and $W$ is a $n\times n$ random orthogonal matrix. We take $N = 100\,000$ and $n = 10$.
By solving the above optimization problem, the goal of ICA is to recover the mixing matrix $W$, up to scaling and permutations invariances; to monitor this we track the Amari distance~\cite{amari1995new} between $X$ and $W^{-1}$.

Figure~\ref{fig:2_cca_mnist} and Figure~\ref{fig:6_ica} show the timings for the Riemannian gradient descent with rolling averaged covariance matrix and the landing algorithm with $\relgrad_B(X)$ in its online and averaged form for the CCA and the ICA experiment respectively. The averaged methods keep track of the covariance matrices during the first pass through the dataset, which is around $3$ sec.~and $0.6$ sec.~respectively, after which they have the exact fully sampled covariance matrices. The online methods have always only the sampled estimate with the batch size of $r=512$. All methods use the fixed step-size $\eta = 0.1$, and the landing methods have $\omega = 1$. In practice, the hyperparameters can be picked by grid-search as is common for stochastic optimization methods.

The online landing method outperforms the averaged Riemannian gradient descent in the online setting in terms of the objective value after only a few passes over the data, e.g.,~at the~$3$ sec.~mark and the~$0.6$ sec.~mark respectively in Figure~\ref{fig:2_cca_mnist} and Figure~\ref{fig:6_ica}, which corresponds to the first epoch, at which point each sample appeared just once. After the first epoch, the rolling avg.~methods get the advantage of the exact fully sampled covariance matrix and, consequently, have better distance $\mathcal{N}(X)$, but at the cost of requiring $\bigO(n^2)$ memory for the full covariance matrix. The online method does not form $B$ and requires only  $\bigO(n(p+r))$ memory. The behavior is also consistent when $p=10$ as shown in Figure~\ref{fig:2_cca_mnist_p10} in the appendices.

\begin{figure}
    \centering
        \includegraphics[width=0.42\textwidth]{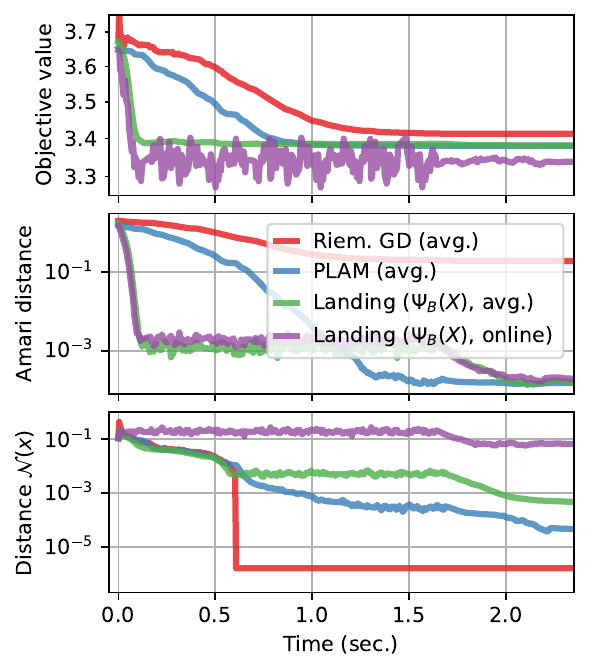}
    \vspace{-1em}
    \caption{Stochastic ICA on the synthetic dataset for $n=10$.\label{fig:6_ica}}
\end{figure}

\section{Conclusion} 
We have extended the theory of the landing method from the Stiefel manifold to the general case of a feasible set defined by smooth equations $h(x)=0$. We have improved the existing analysis by using a smooth merit function, which allows us to also consider situations where we have only random estimates of the manifold. We have showed that the random generalized Stiefel manifold, which is central to problems such as stochastic CCA, ICA, and the GEVP, falls into the category of random manifold feasible set and derived specific bounds for it. 

%\subsubsection*{Author Contributions}
%If you'd like to, you may include  a section for author contributions as is done
%in many journals. This is optional and at the discretion of the authors.

\subsubsection*{Acknowledgments}
% Use unnumbered third level headings for the acknowledgments. All
% acknowledgments, including those to funding agencies, go at the end of the paper.

This work was supported by the Fonds de la Recherche Scientifique-FNRS under Grant no T.0001.23. 
Simon Vary was supported by the FSR Incoming Post-doctoral Fellowship, the Incentive Grant
for Scientific Research (MIS) ``Learning from Pairwise Data'' of the
F.R.S.-FNRS, and by the UK Research and Innovation (UKRI) under the UK government’s Horizon Europe funding guarantee [grant number EP/Y028333/1]. Bin Gao was supported by the Young Elite Scientist Sponsorship Program by CAST and the National Natural Science Foundation of China (grant No.~12288201).

%\newpage 

\section*{Impact Statement} This paper presents theoretical work which aims to advance the field of machine learning. There is no broad impact other than the consequences discussed in the paper.

\bibliography{references}
\bibliographystyle{icml2024}

\newpage
\appendix
\onecolumn

\section{Summary of Retractions on the Generalized Stiefel Manifold}
\label{app:retractions}

For an update to a matrix $X\in\stiefelG{B}$ following a direction in the tangent space $Z\in\mathrm{T}_X\stiefelG{B}$ (see Appendix~\ref{app:canonical_reldescent} for an expression of $\mathrm{T}_X\stiefelG{B}$), there are several ways to compute a retraction. The following asymptotic flop counts provide a simplified picture of computational cost: they do not reflect opportunities for parallelism and assume no structure on matrix $B$.
\begin{itemize}
    \item {The \emph{Polar decomposition}~\citep{Yger2012Adaptive} uses
    \begin{equation*}
        \Retr_{\mathrm{St}_B}(X, Z) = (X+Z)\left( I_p + Z^\top B Z\right)^{-1/2},
    \end{equation*}
    involving the multiplication of $B$ by an $n\times p$ matrix and the computation of the inverse matrix square root of a $p\times p$ matrix, which in naive implementation} amounts to $\bigO(n^2p)$ flops.
    \item{\citet{Mishra2016Riemannian} observed that the aforementioned polar decomposition can be expressed as $UV^\top$ in terms of an SVD-like decomposition of $X+Z = U\Sigma V^\top$, where $U,V$ are orthogonal with respect to $B$-inner product, whose main cost is the eigendecomposition of $(X+Z)^\top B (X+Z)$.}
    \item{Recently, \citet{Sato2019Cholesky} proposed the \emph{Cholesky-QR based retraction}
        \begin{equation*}
            \Retr_{\mathrm{St}_B}(X, Z) = (X+Z)R^{-1},
        \end{equation*}
        where $R\in\bR^{p\times p}$ comes from the Cholesky factorization of $R^\top R = (X+Z)^\top B (X+Z)$. The flops required for the computation, in naive implementation, amount to $\bigO(n^2p)$, which comes from the matrix multiplications. The Cholesky factorization of an $p\times p$ matrix and the inverse multiplication by a small triangular $p\times p$ matrix requires $\bigO(p^3)$ to form and $\bigO(np^2)$ to multiply with.}
\end{itemize}

\section{Additional Experiments and Figures}\label{app:experiments}

For the experiment showed in Fig.~\ref{fig:1_gevp}, we generate the matrix $A\in\bR^{n\times n}$ to have equidistant eigenvalues $\lambda_i(A)\in[1/\kappa_A, 1]$ and $B\in\bR^{n\times n}$ has exponentially decaying eigenvalues $\lambda_i(B)\in[1/\kappa_B, 1]$. We pick the step-size $\eta$ parameter to be $\eta=0.01$ for the Riemannian gradient descent, the landing with $\relgrad_B^\mathrm{R}(X)$, and PLAM, and $\eta=200$ for the landing with $\relgrad_B(X)$ and we run a grid-search with step-sizes $c \eta$, where $c\in [1/4, 1/2, 1, 2, 4, 8]$. The normalizing parameter $\omega$ is chosen to be $\omega=10^5$ for the landing with $\relgrad_B^\mathrm{R}(X)$, $\omega=0.1$ for the landing with $\relgrad_B(X)$, and $\omega=200$ for PLAM.

\begin{figure}[h]
    \centering
    \includegraphics[width=0.45\textwidth]{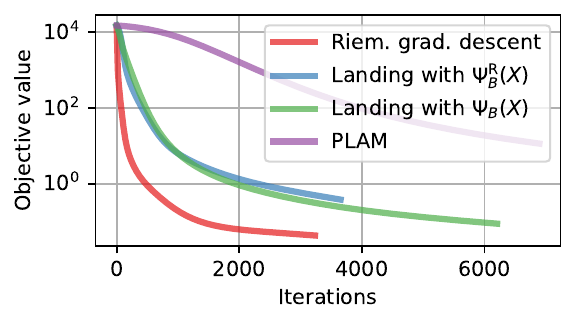}
    \includegraphics[width=0.462\textwidth]{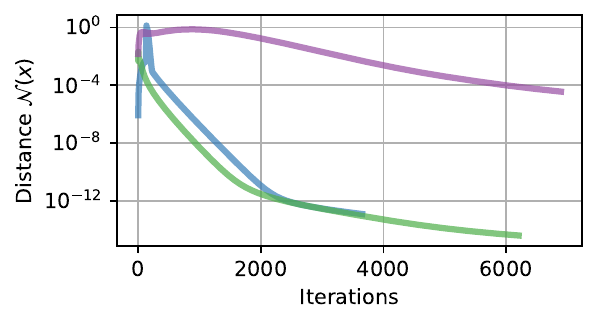}
    \caption{Deterministic computation of the generalized eigenvalue problem with $n = 1000, p = 500$, the condition number of the two matrices $\kappa_B = \kappa_A =100$. Each algorithm is given a time limit of $120$ seconds.}
    \label{fig:1_gevp_iter}
    % \vspace{-1em}
\end{figure}

\begin{figure}[h]
    \centering
    \includegraphics[width=0.45\textwidth]{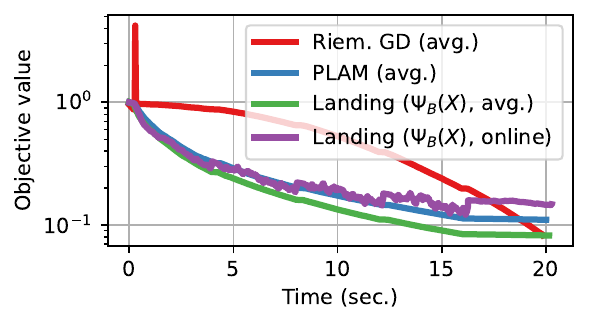}
    \includegraphics[width=0.45\textwidth]{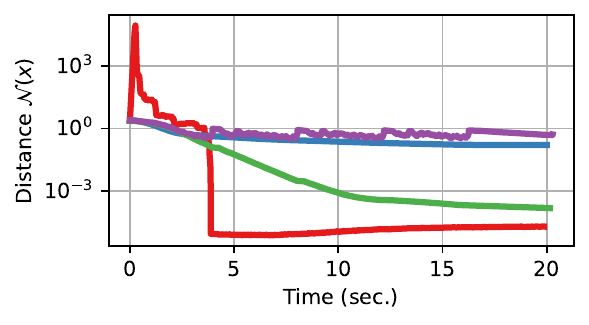}
    \caption{Stochastic canonical correlation analysis on the split MNIST dataset for $p=10$ canonical components.\label{fig:2_cca_mnist_p10}}
\end{figure}

\begin{figure}[h]
    \centering
    \begin{subfigure}[b]{0.45\columnwidth}
         \includegraphics[width=1\textwidth]{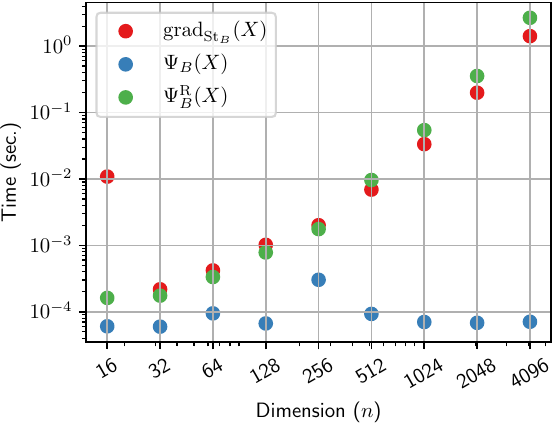}
        \caption{Descent directions.}
        \label{fig:2_cca_mnist_p5}
    \end{subfigure}
    \begin{subfigure}[b]{0.45\columnwidth}
        \includegraphics[width=1\textwidth]{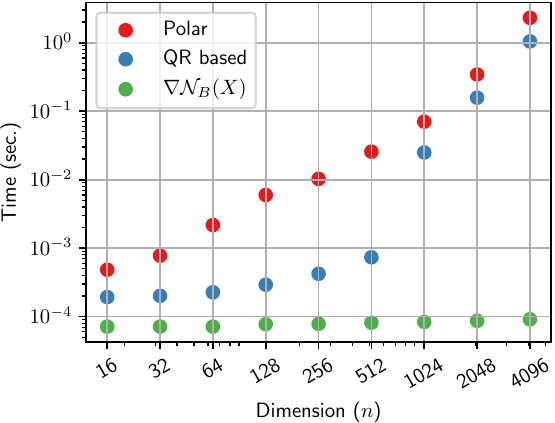}
       \caption{Retractions.\label{fig:3_cost_retractions_cuda}}
    \end{subfigure}
    \caption{Comparison of per-iteration computational time for different problem sizes of the descent directions of algorithms in Fig.~\ref{fig:1_gevp} and the cost of retractions compared to $\nabla \cN(X)$, both in the deterministic setting when $n=p=r$, for which the matrix multiplication in $\relgrad_B(X)$ and $\nabla_\cN(X)$ are at the disadvantage. Computation time of randomly generated $B,X\in\mathbb{R}^{n\times n}$ averaged over $100$ runs with \texttt{CUDA} implementation using \texttt{cupy}. \label{fig:3_cost}}
\end{figure}

\begin{figure}[h]
    \centering
    \includegraphics[width=0.45\textwidth]{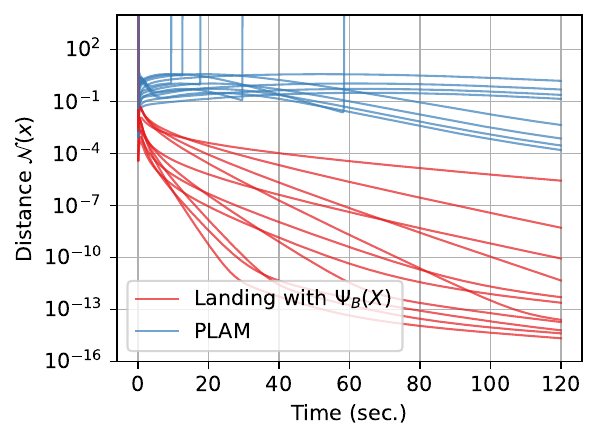}
    \includegraphics[width=0.45\textwidth]{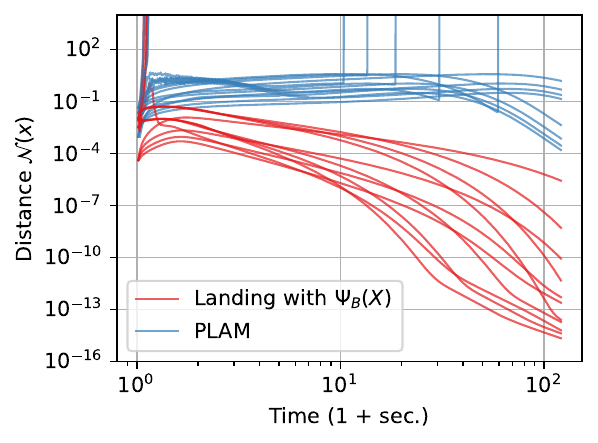}
    \caption{Comparison of the sensitivity to the choice of the step-size $\eta$ and $\omega$ of the landing with $\relgrad_B(X)$ and the PLAM method \citep{Gao2022orthogonalization} in the generalized eigenvalue problem experiment presented in Fig.~\ref{fig:1_gevp} with $n = 1000, p = 500$, and the condition number of the two matrices $\kappa_B = \kappa_A =100$. On the right we show log-log scale to better see the effect in earlier iterations. Both parameters are picked as in the experiment for Fig.~\ref{fig:1_gevp} and multiplied by a scalar from the set $\{0.25, 0.75, 1.25, 1.75\}$ for all possible pair combinations.} 
    \label{fig:4_gevp_sensitivity}
    % \vspace{-1em}
\end{figure}

\begin{figure}[h]
    \centering
    \includegraphics[width=0.45\textwidth]{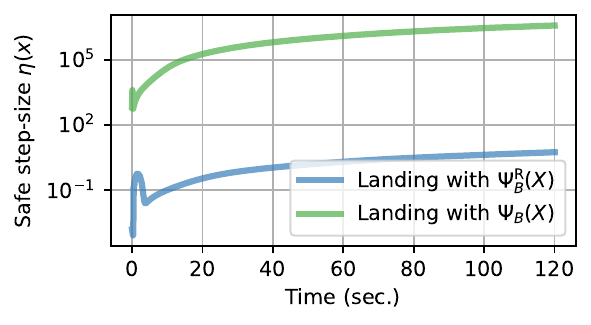}
    \caption{Numerical evaluation of the step-size safeguard $\eta(X)$ in Lemmma~\ref{lemma:safe_step} per time, which ensures that the iterates stay in $\stiefelGeps{B}$, for the two landing methods tested in Fig.~\ref{fig:1_gevp} with the $L_\cN$ bounded for the GEVP as in Lemma~\ref{lemma:gevp_lipschitz}.} 
    \label{fig:1_gevp_safestep}
    % \vspace{-1em}
\end{figure}

\begin{figure}[h]
    \centering
    \begin{subfigure}[b]{0.32\textwidth}
         \includegraphics[width=1\textwidth]{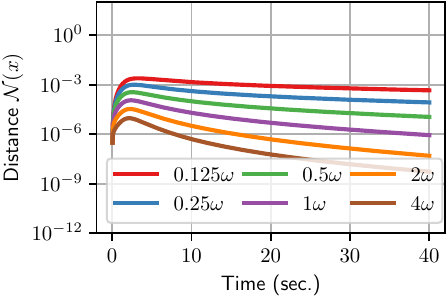}
        \caption{$\eta/8$}
        \label{fig:5_landing_sensitivity_eta0125}
    \end{subfigure}
    \begin{subfigure}[b]{0.32\textwidth}
         \includegraphics[width=1\textwidth]{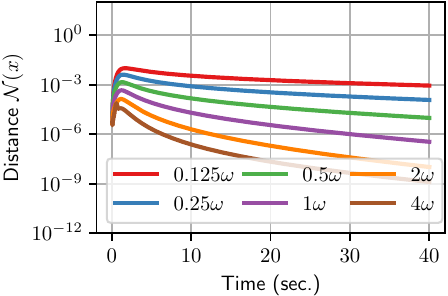}
        \caption{$\eta/4$}
        \label{fig:5_landing_sensitivity_eta025}
    \end{subfigure}
    \begin{subfigure}[b]{0.32\textwidth}
         \includegraphics[width=1\textwidth]{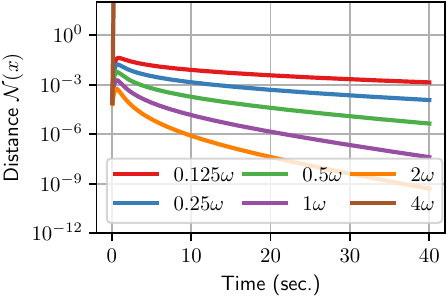}
        \caption{$\eta/2$}
        \label{fig:5_landing_sensitivity_eta05}
    \end{subfigure}
    \begin{subfigure}[b]{0.32\textwidth}
         \includegraphics[width=1\textwidth]{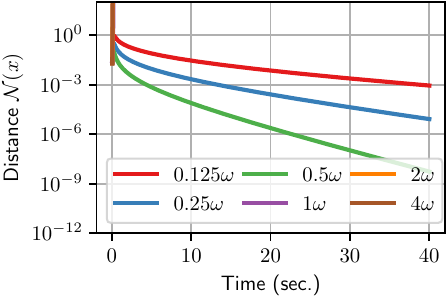}
        \caption{$2 \eta$}
        \label{fig:5_landing_sensitivity_eta05}
    \end{subfigure}
    \begin{subfigure}[b]{0.32\textwidth}
         \includegraphics[width=1\textwidth]{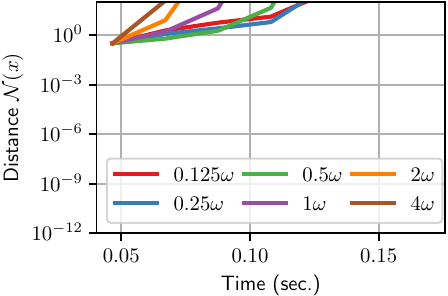}
        \caption{$4 \eta$.}
        \label{fig:5_landing_sensitivity_eta05}
    \end{subfigure}
    \caption{Robustness of the convergence towards the $\stiefelG{B}$ for the landing with $\relgrad_B(X)$ in the experiment for Fig.~\ref{fig:1_gevp} based on the multiplicative perturbations of $\eta$ and $\omega$ parameters with the values from $\{1/8, 1/4, 1/2, 2, 4\}$. } 
    \label{fig:5_landing_sensitivity}
    % \vspace{-1em}
\end{figure}

\section{Proofs for Section~\ref{sec:landing}}

\subsection{Proof of Proposition~\ref{prop:riemannian_gradient}\label{proof:riemannian_gradient}}
\begin{proof}
    The Riemannian gradient can be computed as
   \begin{equation}
        \grad f(x) = \nabla f(x) - \rD h(x)^* \left(\rD h(x)^*\right)^\dagger \nabla f(x), \label{eq:riemannian_gradient}
    \end{equation}
    where $\rD h(x)^* \left(\rD h(x)^*\right)^\dagger$ is the orthogonal projection on the normal space of $\cM_{h(x)}$. 
    It follows from~\eqref{eq:riemannian_gradient} and $\rD h(x) \rD h(x)^* \left(\rD h(x)^*\right)^\dagger=\rD h(x)$ that $\rD h(x)[\grad f(x)]=0$, which implies the first condition in Definition~\ref{def:relative_descent_direction} holds, i.e., $\inner{\grad f(x)}{v} = 0$ for all  $v \in \mathrm{range}(\rD h(x)^*)$. Since $\rD h(x)^* \left(\rD h(x)^*\right)^\dagger \nabla f(x)\in\mathrm{range}(\rD h(x)^*)$, we have
	\begin{align*}
			\| \grad f(x) \|^2 &= \inner{\grad f(x)}{\grad f(x)} \\
			&=  \inner{\grad f(x)}{\nabla f(x) - \rD h(x)^* \left(\rD h(x)^*\right)^\dagger \nabla f(x)}\\
			&=  \inner{\grad f(x)}{\nabla f(x)},
	\end{align*}
    which verifies the second condition with $\rho=1$. It also satisfies the third condition since the critical points are the points of $\mathcal{M}$ where $\mathrm{grad} f$ is zero.
\end{proof}

\subsection{Proof of Lemma~\ref{lemma:safe_step}\label{proof:safe_step}}
\begin{proof}

It is assumed that $\Lambda(x) \neq 0$, otherwise the conclusion of Lemma~\ref{lemma:safe_step} holds regardless of $\eta(x)$. Let $\tilde\eta = \inf \{\eta>0 \,:\, \cN(x - \eta \Lambda(x)) > \frac{\varepsilon^2}2 \}$.  %[[PACOMM I don't think that we are assuming that the safe region is bounded, hence I have to write the next sentence:]]
If $\tilde\eta = \infty$, then the conclusion of Lemma~\ref{lemma:safe_step} trivially holds; hence we now consider that $\tilde\eta < \infty$, i.e., $\tilde\eta$ is the first $\eta$ beyond which $x - \eta \Lambda(x)$ is no longer in the safe region $\cM ^\varepsilon$. Let $\tilde x = x - \tilde\eta \Lambda(x)$, and observe that the line segment from $x$ to $\tilde x$ is in $\cM ^\varepsilon$. Since $\cN$ is $L_\cN$-smooth in $\cM ^\varepsilon$ (Assumption~\ref{ass:h_constraint}), it follows from a standard bound (see, e.g.~\citet{BergerAbsilJungersNesterov2020}) that
\begin{align*}
    \frac{\varepsilon^2}2 = \cN(\tilde x) &\leq  \cN(x) + \inner{\nabla \cN(x)}{-\tilde\eta \Lambda(x)} + \frac{\tilde\eta^2 L_\cN}{2} \|\Lambda(x) \|^2 \\
    & = \cN(x) -\tilde\eta \omega \| \nabla \cN(x)\|^2 + \frac{\tilde\eta^2 L_\cN}{2} \| \Lambda(x)\|^2.
\end{align*}
The function $\bar\cN(\eta) := \cN(x) -\tilde\eta \omega \| \nabla \cN(x)\|^2 + \frac{\tilde\eta^2 L_\cN}{2} \| \Lambda(x)\|^2$ appearing on the right-hand side is a strictly convex quadratic function with $\bar\cN(0) < \frac{\varepsilon^2}2$. Since $\tilde\eta \geq 0$, it follows that $\tilde\eta \geq \eta(x)$, where $\eta(x)$ is the positive solution of $\bar\cN(\eta) = \frac{\varepsilon^2}2 $, whose formula is the one given in the statement of Lemma~\ref{lemma:safe_step}. Hence $x - \eta(x) \Lambda(x)$ is in the line segment from $x$ to $\tilde x$, which is included in $\cM ^\varepsilon$. 
% Rev by PA, 2024-05-20
\end{proof}

\subsection{Proof of Lemma~\ref{lemma:safe_step_lower}\label{proof:safe_step_lower}}
\begin{proof}
    In view of Assumption~\ref{ass:h_constraint}, $\| \nabla \cN(x)\|\geq \constHlow \| h(x)\|$ holds in $\mathcal{M}^\varepsilon$. We proceed to lower bound the numerator of the step-size safeguard $\eta(x)$ in Lemma~\ref{lemma:safe_step} as follows
    \begin{align*}
        \omega \| \nabla \cN(x)\|^2 + &\sqrt{\omega^2 \| \nabla \cN(x)\|^4 + L_\cN \|\Lambda(x)\|^2 (\varepsilon^2 -\| h(x)\|^2)} \nonumber \\
        &\geq  \omega \constHlow^2 \| h(x)\|^2 + \sqrt{\omega^2 \constHlow^4 \| h(x)\|^4  + L_\cN \| \relgrad(x)\|^2 \left(\varepsilon^2 - \| h(x) \|^2 \right)} \\
        &\geq \omega \constHlow^2 \| h(x)\|^2 \left(1+\frac{1}{\sqrt2}\right)+ \frac{1}{\sqrt2}\| \relgrad(x)\| \sqrt{L_\cN\left(\varepsilon^2 - \| h(x)\|^2\right)}\\
        &\geq \sqrt{\frac{L_\cN}{2}}\|\relgrad(x)\| (\varepsilon - \| h(x)\| ) + \left(1+\frac{1}{\sqrt2}\right)\omega \constHlow^2 \| h(x)\|^2
    \end{align*}
    where the first inequality comes from using bounds from Assumption~\ref{ass:h_constraint}, the second inequality comes from $\sqrt{a+b} \geq (\sqrt{a} + \sqrt{b})/\sqrt2$ for $a,b\geq 0$, and the final inequality from the fact that $\sqrt{a - b} \geq \sqrt{a} - \sqrt{b}$ for $a,b\geq 0$ and $a \geq b$. As a result we have that $\eta(x)$ in Lemma~\ref{lemma:safe_step} is lower-bounded by
    \begin{equation}  
        \eta(x) \geq \frac{\sqrt{\frac{L_\cN}{2}}\|\relgrad(x)\| (\varepsilon - \| h(x)\| ) + \left(1+\frac{1}{\sqrt2}\right)\omega \constHlow^2 \| h(x)\|^2}{L_\cN \left(\|\relgrad(x)\|^2 + \omega^2 \constH^2 \| h(x)\|^2 \right)}, \label{eq:safe_step_lower_ineq1}
    \end{equation}
    using the fact that $\| \Lambda(x)\|^2 = \| \relgrad(x)\|^2 + \omega^2 \| \nabla \cN(x)\|^2$ and $\| \nabla \cN(x)\|^2\leq \constH^2 \| h(x)\|^2$. 
    
    The right-hand side of \eqref{eq:safe_step_lower_ineq1} takes the form
    \begin{equation*}
        \underline\eta:= \frac{a P (\varepsilon - H) + b H^2}{c P^2 + d H^2},
    \end{equation*}
    where $a = \sqrt{\frac{L_\cN}{2}}$, $b = \left(1+\frac{1}{\sqrt2}\right)\omega \constHlow^2$, $c = L_\cN$, and $d = L_\cN \omega^2 \constH^2$ are constants and $P = \| \Psi(x)\|$, $H = \| h(x)\|$ are variables in bounded intervals: $0 \leq P \leq \constRgrad$ and $0 \leq H \leq \varepsilon$.

    Pick $\alpha \in (0,1)$. There are two cases, either $H \geq \alpha \varepsilon$ or $H < \alpha \varepsilon$.

    When $H \geq \alpha \varepsilon$, we have 
    \begin{equation}
        \underline\eta \geq \frac{b \alpha^2 \varepsilon^2}{c \constRgrad^2 + d \varepsilon^2}. %= \frac{\left(1+\frac{1}{\sqrt2}\right)\alpha^2\varepsilon^2}{L_\cN \constRgrad^2 + L_\cN \omega^2 \constH^2\varepsilon^2}.
    \end{equation}

    For the second case, when $H < \alpha \varepsilon$:
    \begin{equation}
        \underline\eta \geq \frac{a P (1-\alpha) \varepsilon + b H^2}{c P^2 + d H^2}.  \label{eq:safe_step_lower_ineq1b}
    \end{equation}
    
    We can again distinguish two cases to further lower bound \eqref{eq:safe_step_lower_ineq1b}. When $P\geq 1$, we have 
    \[
        \eqref{eq:safe_step_lower_ineq1b} \geq \frac{a (1-\alpha)\varepsilon}{c \constRgrad^2 + d \varepsilon^2}.
    \]
    When $P \leq 1$, we get
    \begin{align*}
        \eqref{eq:safe_step_lower_ineq1b} &\geq \frac{a P^2 (1-\alpha) \varepsilon + b H^2}{c P^2 + d H^2} \\
    &= \frac{\frac{a (1-\alpha) \varepsilon}{c} c P^2  + \frac{b}{d} d H^2}{c P^2 + d H^2} \\
    &\geq \min\{ \frac{a (1-\alpha) \varepsilon}{c}, \frac{b}{d} \} \frac{c P^2 (1-\alpha) \varepsilon + d H^2}{c P^2 + d H^2} \\
    &= \min\{ \frac{a (1-\alpha) \varepsilon}{c}, \frac{b}{d} \}.
    \end{align*}
    
    Putting all the cases together, we are left with the minimum of four terms
    \begin{align*}
        \underline\eta &\geq \min\left\{ \frac{b \alpha^2 \varepsilon^2}{c \constRgrad^2 + d \varepsilon^2}, \frac{a (1-\alpha)\varepsilon}{c \constRgrad^2 + d \varepsilon^2}, \frac{a (1-\alpha) \varepsilon}{c}, \frac{b}{d} \right\} \\
        &\geq\min\left\{ \frac{\omega \constHlow^2 \alpha^2 \varepsilon^2}{ L_\cN \left( \constRgrad^2 + \omega^2 \constH^2 \varepsilon^2 \right)}, \frac{(1-\alpha)\varepsilon}{\sqrt{2L_\cN}\left(  \constRgrad^2 + \omega^2 \constH^2 \varepsilon^2 \right)}, \frac{(1-\alpha) \varepsilon}{\sqrt{2L_\cN}}, \frac{1}{\omega L_\cN}\left( \frac{\constHlow}{\constH} \right)^2 \right\},
    \end{align*}
    where in the second line we also further lower bounded the first and the last terms in the minimum by using that $1 \leq \left(1+\frac{1}{\sqrt2}\right)$.    
\end{proof}

\subsection{Proof of Lemma~\ref{lemma:fletcher_inner}\label{proof:fletcher_inner}}
\begin{proof}
    The inner product has two parts
    \begin{align}
        \inner{\nabla \cL(x)}{ \Lambda(x)} &= \rD \cL(x)[\Lambda(x)] \nonumber \\
        = \rD \cL(x)&[\relgrad(x)] + \normCoef \rD \cL(x)[\nabla \cN(x)]. \label{eq:lemma:fletcher_inner1}
    \end{align}
    We expand the first term of the right hand side of \eqref{eq:lemma:fletcher_inner1} as
    \begin{align}
        \rD \cL(x) [\Psi(x)] &= \inner{\nabla f(x)}{\Psi(x)}  - \inner{(\rD h(x)^*)^\dagger \nabla f(x)}{\rD h(x) \Psi(x)} \nonumber \\
        & \qquad - \inner{\rD \lagrM(x) [\Psi(x)]}{h(x)} + 2\beta \inner{\nabla \cN(x)}{\Psi(x)} \nonumber \\
        & = \inner{\nabla f(x)}{\Psi(x)} - \inner{\rD \lagrM(x) [\Psi(x)]}{h(x)} \label{eq:lemma:fletcher_inner2c}
    \end{align}
    where we use that $\nabla \| h(x)\|^2 = 2\nabla \cN(x)$ and that the second and the third term are zero due to the orthogonality of $\relgrad(x)$ with the range of $\rD h(x)^*$. % $\inner{\rD h(x)^* (\rD h(x)^*)^\dagger \nabla f(x)}{\Psi(x)} = 0$ and $\inner{\rD h(x)^* h(x)}{\Psi(x)} = 0$. 
    We expand the second term of the right hand side of in \eqref{eq:lemma:fletcher_inner1} as
    \begin{align}
        \rD \cL(x)[\nabla \cN(x)] &= \inner{\nabla f(x)}{\nabla \cN(x)} - \inner{(\rD h(x)^*)^\dagger \nabla f(x)}{\rD h(x) \nabla \cN(x)} \nonumber \\
        & \qquad - \inner{\rD \lagrM(x) [\nabla \cN(x)]}{h(x)} + 2\beta \| \nabla \cN(x) \|^2 \nonumber \\
        &=\inner{(I_n - \rD h(x)^*(\rD h(x)^*)^\dagger ) \nabla f(x)}{\nabla \cN(x)} \nonumber \\
        & \qquad - \inner{\rD \lagrM(x) [\nabla \cN(x)]}{h(x)} + 2\beta \| \nabla \cN(x)\|^2 \nonumber \\
        &= - \inner{\rD \lagrM(x) [\nabla \cN(x)]}{h(x)} + 2\beta \| \nabla \cN(x)\|^2, \label{eq:lemma:fletcher_inner3c}
    \end{align}
    where in the second equality we move the adjoint $\rD h(x)^*$ in the second inner product to the left side and join it with the first inner product. The third equality comes from the fact that the projection of $\nabla f(x)$ on the null space of $\mathrm{D}h(x)$ and $\nabla\cN(x) = \rD h(x)^* h(x)$ are orthogonal.

    Joining the two components \eqref{eq:lemma:fletcher_inner2c} and \eqref{eq:lemma:fletcher_inner3c} together we get
    \begin{align*}
        \langle \nabla \cL(x),\,\Lambda(x)\rangle &= \inner{\nabla f(x)}{\Psi(x)} - \inner{\rD \lagrM(x) [\Lambda (x)]}{h(x)} + 2\beta \omega\| \nabla \cN(x) \|^2
        \\
        &\geq \rho \| \relgrad(x)\|^2 - \constLagrM \left(\| \relgrad(x) \| + \omega \| \nabla \cN(x) \|\right)\| h(x) \| + 2\beta \omega \|\nabla \cN(x)\|^2\\
        &\geq \rho \| \relgrad(x)\|^2 + \omega( 2\beta \constH - \constLagrM ) \constH \| h(x)\|^2 - \constLagrM \| \relgrad(x) \| \| h(x) \|\\
        &\geq \rho \| \relgrad(x)\|^2 + \omega( 2\beta \constH - \constLagrM ) \constH \| h(x)\|^2 - \frac{\constLagrM}{2} \left( \alpha \| \relgrad(x) \|^2 + \alpha^{-1} \| h(x) \|^2 \right)\\
        &\geq \left(\rho - \frac{\constLagrM}{2} \alpha\right) \| \relgrad(x)\|^2 + \left( 2 \omega\beta \constH^2 -\omega \constH \constLagrM - \alpha^{-1} \frac{\constLagrM}{2} \right) \| h(x)\|^2\\
        %&\geq \frac\rho2 \| \relgrad(x)\|^2 + ( 2\beta \constH^2 -\omega \constH \constLagrM -  \frac{2\constLagrM^2}{\rho} ) \| h(x)\|^2 \\
        &\geq \frac\rho2 \left( \| \relgrad(x)\|^2 + \| h(x)\|^2 \right)
    \end{align*}
    where the first inequality comes from $\inner{\nabla f(x)}{\relgrad(x)}\geq \rho \|\relgrad(x)\|^2$ in Definition~\ref{def:relative_descent_direction} combined with the bound $\sup_{x\in\cM^\varepsilon}\|\rD \lambda(x) \|\leq \constLagrM$ and the triangle inequality, the second inequality comes from bounding $\|\nabla \cN(x)\|\leq \constH\| h(x) \|$ using Assumption~\ref{ass:h_constraint} and rearranging terms, the third inequality comes from using the AG-inequality $\sqrt{ab}\leq (a+b)/2$ with $a = \alpha \| h(x) \|^2$ and $b = \alpha^{-1} \| \relgrad (x)\|^2$ for an arbitrary $\alpha>0$, in the fourth inequality we only rearrange terms, and finally, in the fifth inequality we choose $\alpha = \rho/\constLagrM$ and use that $\beta \geq (\frac{\rho}{4\constH^2} + \frac{\omega\constLagrM}{2\constH} + \frac{\constLagrM^2}{4 \rho\constH^2})/\omega$.
\end{proof}

\subsection{Proof of Theorem~\ref{thm:landing_convergence}\label{proof:landing_convergence}}
\begin{proof}
    Due to $x_0\in\cM^\varepsilon$ and the step-size $\eta$ being smaller than the bound on the step-size safeguard in Lemma~\ref{lemma:safe_step_lower}, we have that the segment $[x^k, x^{k+1}]$ is included in $\mathcal{M}^\varepsilon$ for all $k$. By $L_{\mathcal{L}}$-smoothness of Fletcher's augmented Lagrangian in $\mathcal{M}^\varepsilon$, we can expand
    \begin{align}
        \cL(x^{k+1}) &\leq \cL(x^k) - \eta \inner{\Lambda(x^k)}{\nabla \cL(x^k)} + \frac{L_\cL\eta^2}{2}\| \Lambda(x^k) \|^2 \\
        &\leq \cL(x^k) - \frac{\eta\rho}{2} \left( \| \Psi(x^k)\|^2 + \| h(x^k) \|^2 \right) + \frac{L_\cL\eta^2}{2}\| \Lambda(x^k) \|^2 \\
        &\leq \cL(x^k) - \frac{\eta}{2}\Big( \left(\rho - L_\cL\eta\right)\| \Psi(x^k) \|^2 +  \left(\rho - \eta L_\cL \Rev{\omega^2} \constH^2 \right) \| h(x^k)\|^2 \Big), \label{eq:landning_convergence1}
    \end{align}
    where in the second inequality we used the results of Lemma~\ref{lemma:fletcher_inner}, and in the third inequality we use the bound on $\| \nabla \cN(x)\| \leq \constH \| h(x)\|$ by Assumption~\ref{ass:h_constraint}. By the step-size $\eta < \min{\left\{ \frac{\rho}{2L_\cL}, \frac{\rho}{2L_\cL \Rev{\omega^2}\constH^2}\right\}}$ we have
    \begin{equation}
        \frac{\eta\rho}{4} \| \relgrad(x^k) \|^2 + \frac{\eta\rho}{4}\| h(x^k)\|^2 \leq \cL(x^k) - \cL(x^{k+1}).
    \end{equation}
    Telescopically summing the first $K+1$ terms gives
    \begin{align*}
        \frac{\eta\rho}{4} \sum_{k=0}^K \| \relgrad(x^k) \|^2 + \frac{\eta\rho}{4} \sum_{k=0}^K \| h(x^k)\|^2 \leq \cL(x^0) - \cL(x^{K+1}) \leq \cL(x^0) - \cL^*, 
    \end{align*}
    which implies that the inequalities hold individually also
    \begin{equation*}
        \frac{\eta\rho}{4} \sum_{k=0}^K \| \relgrad(x^k) \|^2 \leq  \cL(x^0) - \cL^* \qquad\text{and}\qquad \frac{\eta\rho}{4} \sum_{k=0}^K \| h(x^k)\|^2 \leq \cL(x^0) - \cL^*.
    \end{equation*}
\end{proof}

\subsection{Proof of Theorem~\ref{thm:landing_stochastic_convergence}\label{proof:landing_stochastic_convergence}}
\begin{proof}
    Let $x^{k+1} = x^{k} - \eta_k \tilde \Lambda (x^k)$, where we denote by $\tilde \Lambda(x^k) = \Lambda(x^k) + \tilde E(x^k, \Xi^k)$ the unbiased estimator of the landing update, and we assume that the line segment between the iterates remain within $\cM^{\varepsilon}$. By $L_{\mathcal{L}}$-smoothness of Fletcher's augmented Lagrangian inside $\cM^\varepsilon$, we have
    \begin{align*}
        \bE_{\Xi^k}&\left[\cL(x^{k+1})\right] \leq \bE_{\Xi^k}\Big[\cL(x^k) - \eta_k \inner{\tilde \Lambda(x^k)}{\nabla \cL(x^k)} + \frac{L_\cL\eta_k^2}{2}\| \tilde \Lambda(x^k) \|^2 \Big]\\
        & \leq \cL(x^k) - \eta_k \inner{\Lambda(x^k)}{\nabla \cL(x^k)} + \frac{L_\cL\eta_k^2}{2}\left( \| \Lambda(x^k) \|^2 + \gamma^2  \right) \\
        &\leq \cL(x^k) - \frac{\eta_k\rho}{2} \left( \| \Psi(x^k)\|^2 + \| h(x^k) \|^2 \right) + \frac{L_\cL\eta_k^2}{2} \left( \| \Lambda(x^k) \|^2 + \gamma^2 \right)\\
        &\leq \cL(x^k) + \frac{L_\cL\eta_k^2}{2} \gamma^2 -\frac{\eta_k}{2}\left( \left(\rho - L_\cL\eta_k\right) \| \Psi(x^k) \|^2 + \left(\rho - \eta_k L_\cL \Rev{\omega^2}\constH^2 \right) \| h(x^k)\|^2 \right),
    \end{align*}
    where the first inequality comes from taking an expectation of a bound akin the first bound of~\autoref{proof:landing_convergence}, in the second inequality we take the expectation inside the inner product using the fact that $\tilde E(x^k, \Xi^k)$ is zero-centered and has bounded variance, the third inequality comes as a consequence of Lemma~\ref{lemma:fletcher_inner}. The last inequality comes as a consequence of $\Lambda(x^k)$ having two orthogonal components and rearranging terms in the same way as in \eqref{eq:landning_convergence1}.
    Note that by $\mathbb{E}_{\Xi^k}$ we denote expectation only with respect to the last random realization $\Xi^k$.

    By the step-size being smaller than $\eta_k \leq \eta_0 < \frac{\rho}{2L_\cL} \min\left\{1,(\Rev{\omega}  \constH)^{-2}\right\}$ we have that
    \begin{equation}
        \frac{\eta_k\rho}{4} \| \relgrad(x^k) \|^2 + \frac{\eta_k\rho}{4}\| h(x^k)\|^2 \leq \cL(x^k) - \bE_{\Xi^k} \left[\cL(x^{k+1})\right] + \frac{L_\cL\eta_k^2}{2} \gamma^2\label{eq:stochastic_decrease1}.
    \end{equation}
    Taking the expectation of \eqref{eq:stochastic_decrease1} with respect to the whole past random realizations $\Xi^0, \dots, \Xi^k$, denoted for short simply as $\mathbb{E}$, yields
    \begin{equation}
        \bE\left[\frac{\eta_k\rho}{4} \| \relgrad(x^k) \|^2 + \frac{\eta_k\rho}{4}\| h(x^k)\|^2 \right] \leq \bE[\cL(x^k)] - \bE\left[\bE_{\Xi^k} \left[\cL(x^{k+1})\right]\right] + \frac{L_\cL\eta_k^2}{2} \gamma^2.\label{eq:stochastic_decrease2}
    \end{equation}
    Since $x^{k+1} = x^k - \eta_k \tilde\Lambda(x^k)$, we have that $\bE\left[\bE_{\Xi^k}[\cdot]\right] =\bE[\cdot]$, and we can telescopically sum the first $K+1$ terms of \eqref{eq:stochastic_decrease1} for $k=0, 1, \ldots, K$:
    \begin{align}
        \frac{\rho}{4}\left( \sum_{k=0}^{K} \eta_k \bE\left[ \| \relgrad(x^k) \|^2\right] +  \sum_{k=0}^{K} \eta_k \bE \left[ \| h(x^k)\|^2 \right] \right)
        &\leq \cL(x^0) - \bE \left[\cL(x^{K+1})\right] + \frac{L_\cL\eta_0^2 \gamma^2}{2} \sum_{k=0}^{K} (1+k)^{-1} \label{eq:landing_stochastic_convergence_sum}\\
        &\leq \cL(x^0) - \cL^* + \frac{L_\cL\eta_0^2\gamma^2}{2}\left( 1+ \log(K+1)\right) \nonumber
    \end{align}
    which implies that the inequalities hold also individually
    \begin{align*}
        \inf_{k\leq K} \bE \left[\| \relgrad(x^k) \|^2\right] & \leq 4\frac{\cL(x^0) - \cL^*}{\rho \eta_0 \sqrt{K}} + 2 \frac{\eta_0 L_\cL \gamma^2}{\rho} \left(\frac{1+ \log(K+1)}{\sqrt{K}}\right), \\
        \inf_{k\leq K} \bE \left[\| h(x^k) \|^2\right] & \leq 4\frac{\cL(x^0) - \cL^*}{\rho  \eta_0 \sqrt{K}} + 2 \frac{\eta_0 L_\cL \gamma^2}{\rho} \left(\frac{1+ \log(K+1)}{\sqrt{K}}\right),
        %\leq \frac{4}{\rho \omega^2 \eta_0 \sqrt{K}}\left( \cL(x^0) - \cL^*  + \frac{\eta^{2}_0 L_\cL \gamma^2}{2} \left( 1+ \log(K\Rev{+1})\right) \right),
    \end{align*}
    where we used that $\inf_{k\leq K}  \bE\| \relgrad(x^k) \|^2\leq \sum_{k=0}^K \eta_k  \bE\|\relgrad(x^k)\|^2 \left( \sum_{k=0}^K \eta_k\right)^{-1}$ and the fact that $\sum_{k\leq K} \eta_k \geq \eta_0 \sqrt{K}$.
\end{proof}

\section{Proofs for Section~\ref{sec:generalized_stiefel}}

\subsection{Specific forms of $\rD h(x), \lambda(X)$ for $\stiefelG{B}$}\label{app:specific_forms}

We begin by showing the specific form of the formulations derived in the previous section for the case of the generalized Stiefel manifold. Let $h: \mathbb{R}^{n\times p} \to \mathrm{sym}(p): X \mapsto X^\top B X - I_p$, where letting $\mathrm{sym}(p)$ be the codomain is essential for Assumption~\ref{ass:h_constraint} to hold. Differentiating the generalized Stiefel constraint yields $\rD h(X)[V] = X^\top B V + V^\top B X$ and the adjoint $\mathrm{D}h(X)^*: \mathrm{sym}(p) \to \mathbb{R}^{n\times p}$ is derived as
\begin{align}
    \inner{\rD h(X)^* [V]}{W} = \inner{V}{\rD h(X)[W]} &=\inner{V}{W^T B X + X^T B W} = \inner{2BXV}{W},
\end{align}
as such we have that $\rD h(X)^*[V] = 2BXV$. Consequently
\begin{equation}
    \rD h(X) \rD h(X)^*[V] = 2VX^\top B^2 X + 2X^\top B^2 X V,
\end{equation}
and the Lagrange multiplier $\lambda(X)$ is defined in the case of the generalized Stiefel manifold as the solution to the following Lyapunov equation
\begin{equation}
    2\lambda(X) X^\top B^2X + 2X^\top B^2 X \lambda(X) = X^\top B \nabla f(X) + \nabla f(X)^\top B X. \label{eq:lyapunov_system}
\end{equation}
Importantly, due to $\lagrM(X)$ being the unique solution to the linear equation, which is ensured by $BX$ having a full rank since $X$ is in the $\varepsilon$-safe region, and by the the linear operator being smooth in $X$, since $\nabla f(X)$ is smooth, we have that $\lagrM(X)$ is invertible and smooth with respect to $X$. Thus, as a smooth function defined over a compact set $\stiefelGeps{B}$, its operator norm is bounded: $\sup_{X\in\stiefelGeps{B}}\| \rD \lagrM(X) \|_F \leq \constLagrM$ as required by Assumption~\ref{ass:lagrange_mult}.

\subsection{Proof of Proposition~\ref{prop:gen_stiefel_constants}\label{proof:gen_stiefel_constants}}
\begin{proof}
    For $\| X^\top B X - I_p\|_F \leq \varepsilon$, let  $X = U\Sigma V^\top$ be a truncated singular value decomposition of $X$, and $Q D Q^\top$ be an eigendecomposition of $B$. We then have
    \begin{align}
        \varepsilon^2 \geq \| X^\top B X - I_p \|_{\mathrm{F}}^2 &= \| \Sigma U^\top Q D (U^\top Q)^\top \Sigma - I_p\|_{\mathrm{F}}^2 \label{eq:gen_stiefel_constants1}
    \end{align}
    where $\beta_i, \sigma_i$ are the positive eigenvalues of $B$ and the singular values of $X$ respectively in  decreasing order. 

    Denote $P = Q^\top U \in\mathbb{R}^{n\times p}$ that forms an orthogonal frame $P^\top P = I_p$. The bound in \eqref{eq:gen_stiefel_constants1} implies 
    \begin{equation}
        \varepsilon^2 \geq \sum_{i=1}^p \left( \sigma_i^2 \left(P^\top D P\right)_{ii} - 1 \right)^2,
    \end{equation}
    where $\left(P^\top D P\right)_{ii}$ marks the $i^{th}$ diagonal entry of the matrix $P^\top D P$. Consequently, we have that
    \begin{equation}
        1-\varepsilon \leq \sigma_i^2 \left( P^\top D P\right)_{ii} \leq 1+\varepsilon \label{eq:gen_stiefel_constants2}
    \end{equation}
    for all $i=1,\ldots, p$. We can bound
    \begin{equation}
        \beta_n  = \inf_{\|x\|_2 = 1} x^\top D x \leq \left( P^\top D P\right)_{ii} \leq \sup_{\| x \|_2 = 1} x^\top D x =\beta_1, \label{eq:gen_stiefel_constants3}
    \end{equation}
    since $P^\top P = I_p$. The inequality in \eqref{eq:gen_stiefel_constants2} combined with \eqref{eq:gen_stiefel_constants3} implies that 
    \begin{equation}
        \sqrt{(1-\varepsilon)/\beta_1} \leq \sigma_i \leq \sqrt{(1+\varepsilon)/\beta_{n}}.\label{eq:gen_stiefel_eval_bounds}
    \end{equation} 
    By the lower and the upper bounds on singular values of a matrix product, the above bound gives that the singular values of $\rD h(X)^* = 2BX$ are in the interval $[2 \sqrt{(1-\varepsilon)\beta_n\kappa_B^{-1}}, 2\sqrt{(1+\varepsilon)\beta_1\kappa_B}]$ which in turn gives the constants $\constH, \constHlow$.
\end{proof}

\subsection{Proof of Proposition~\ref{prop:generalized_stiefel_relative_descent}\label{proof:generalized_stiefel_relative_descent}}
\begin{proof}
    First consider $\Psi_B(X)$. For ease of notation we denote $G = \nabla f(X)\in\bR^{n\times p}$. The first property Definition~\ref{def:relative_descent_direction} \emph{(i)} comes from
    \begin{equation}
        \inner{\sk(GX^\top B) BX}{BX S}= 0,
    \end{equation}
    which holds for a symmetric matrix $S$, since a skew-symmetric matrix is orthogonal in the Frobenius inner product to a symmetric matrix,

    The second property \emph{(ii)} is a consequence of the following
	\begin{equation}
		\inner{\relgrad_B(X)}{G} = \inner{\sk(G X^T B)BX}{G}
        =\| \sk(G X^T B) \|_{\mathrm{F}}^2 \geq \frac{1}{(1+\varepsilon) \beta_1\kappa_B }\| \relgrad_B(X) \|_{\mathrm{F}}^2,\label{eq:aligned_grad1}
	\end{equation}
    where we use the bounds on $\| BX \|_2 \leq \sqrt{(1+\varepsilon)\beta_1\kappa_B}$ derived in the proof of Proposition~\ref{prop:gen_stiefel_constants}.
	
	To show the third property \emph{(iii)}, we first consider a critical point $X\in\stiefelG{B}$, for which it must hold that $G$ is in the image of $\mathrm{D}h(X)^*$, i.e.,
	\begin{equation}
		G = BXS, \label{eq:critical_point1}
	\end{equation}
	for some $S\in\sym(p)$ and that $X^\top B X = I_p$ by feasibility. We have that at the critical point defined in \eqref{eq:critical_point1}, the relative ascent direction is
	\begin{align}
		\relgrad_B(X) &= \sk(G X^\top B)BX = \sk(BX S X^\top B)BX = 0,
	\end{align}
	where the second equality is the consequence of \eqref{eq:critical_point1} and the third equality comes from the fact that $BX S X^\top B$ is symmetric.
	
	To show the other side of the implication, that $\relgrad_B(X)=0$ combined with feasibility imply that $X$ is a critical point, we consider 
	\begin{align}
		0 = \relgrad_B(X) = \sk(G X^\top B)BX = GX^\top B^2 X - BXG^\top BX \label{eq:critical_point3}
	\end{align}
	which, since $X^\top B^2 X\in\bR^{p\times p}$ is invertible, is equivalent to
	\begin{equation}
		G =  BXG^\top BX \left(X^\top B^2 X\right)^{-1}. \label{eq:critical_point3b}
	\end{equation}

     It remains to show that the factor $G^\top BX \left(X^\top B^2 X\right)^{-1}$ in~\eqref{eq:critical_point3b} is symmetric in order to get~\eqref{eq:critical_point1}. To this end, multiply~\eqref{eq:critical_point3} on the left by $(X^\top B^2 X)^{-1} X^\top B$ and on the right by $(X^\top B^2 X)^{-1}$ and rearrange the terms. 

    For the other choice of relative gradient $\Psi_B^\mathrm{R}(X) = \sk(B^{-1}GX^\top)BX$, letting $M = B^{-1}GX^\top$, we find 
    \begin{align}
        \langle \Psi_B^\mathrm{R}(X), G\rangle &= \langle \sk(M), BMB\rangle \\
        &=\langle \sk(M), \sk(BMB)\rangle\\
        &=\langle \sk(M), B\sk(M)B\rangle\\
        &\geq \|\sk(M)\|_{\mathrm{F}}^2\beta_n^2
    \end{align}
    and similarly as in \eqref{eq:aligned_grad1}, it holds $\|\Psi_B^\mathrm{R}(X)\|^2\leq \|\sk(M)\|_{\mathrm{F}}^2(1+\varepsilon)\beta_1\kappa_B$ which in turn leads to $\langle \Psi_B^\mathrm{R}(X), G\rangle \geq \frac{\beta_n}{1+\varepsilon}\|\Psi^\mathrm{R}_B(X)\|^2$
\end{proof}

\subsection{Lipschitz constants for the GEVP}
\begin{lemma}[Lipschitz constants for the generalized eigenvalue problem]\label{lemma:gevp_lipschitz}
    Let  $f= -\frac12 \Tr (X^\top A X)$ and $\cN(X) = \frac12 \| X^\top B X - I_p\|_F^2$ as in the optimization problem corresponding to the generalized eigenvalue problem. We have that, for $X\in\stiefelGeps{B}$, a Lipschitz constant for $\nabla \cN$ is $L_\cN= 2 \beta_1\left(\varepsilon + 2(1+\varepsilon)\kappa_B\right)$ and the Lipschitz constant for $\nabla f$ is $L_f = \alpha_1$ where $\alpha_1$ is the largest eigenvalue of $A$.
\end{lemma}
\begin{proof}
    Take $X, Y\in\stiefelG{B}$, we have that $\nabla \cN(X) = 2BX (X^\top B X - I_p)$, thus
    \begin{align}
        \nabla \cN(X) - \nabla \cN(Y) &= 2B\left( X(X^\top B X - I_p) - Y (Y^\top B Y - I_p)\right) \\
        &= 2B \left( (X-Y)(X^\top B X - I_p) + Y\left((X^\top BX - Y^\top B Y) \right)\right) \\
        & = 2B \left( (X-Y)(X^\top B X - I_p) + Y\left( (X-Y)^\top B X + Y^\top B (X-Y) \right)\right).
    \end{align}
    Taking the Frobenius norm and by the triangle inequality we get
    \begin{align}
        \| \nabla \cN(X) - \nabla \cN(Y)\| &\leq 2\left(\| B(X-Y)(X^\top B X - I_p)\| + \| BY (X-Y)^\top BX \| + \| BY Y^\top B (X-Y)\|\right)\\
        &\leq 2 \left(\left\|X-Y\right\| \left\|B (X^\top B X - I_p)\right\|_2 + \left\| X-Y\right\| \left\| BY BX \right\|_2 + \left\| X-Y\right\| \left\| BY Y^\top B\right\|_2 \right) \\
        &\leq 2 \| X - Y \|  \left( \| B \|_2 \| X^\top B X - I_p \| + \| B \|_2^2 \| X \|_2 \| Y\|_2 + \| B\|_2^2 \| Y\|_2^2\right)\\
        &\leq 2 \beta_1\left(\varepsilon + 2(1+\varepsilon)\kappa_B\right) \| X-Y \|,
    \end{align}
    where for the second inequality we used that $\| A B\| \leq \|A\|_2 \|B\|$, the third inequality comes from submultiplicativity of the induced $\ell_2$-norm for matrices, and the fourth inequality comes from $X,Y\in\stiefelGeps{B}$ for which we have that $\| X\|_2 \leq \sqrt{(1+\varepsilon)/\beta_n}$, as in \eqref{eq:gen_stiefel_eval_bounds}, and the same for $Y$.

    When $f = \frac12 \Tr(X^\top A X)$, we have that $\| \nabla f(X) - \nabla f(Y)\|\leq \| A\|_2 \| X-Y\|$.
\end{proof}

\subsection{Proof of Proposition~\ref{prop:variance_landing}\label{proof:variance_landing}}
\begin{proof}
We start by deriving the bound on the variance of the normalizing component $\nabla \cN(X)$. Consider $U$ and $V$ to be two independent random matrices taking i.i.d.~values from the distribution of $B_\zeta$ with variance $\sigma_B^2$. We have that
\begin{equation}
	\mathrm{Var}\left[UX(X^\top VX - I_p)\right] = \mathbb{E}_{U, V}\left[\|UX(X^\top VX - I_p) - BX(X^\top BX - I_p) \|^2\right].
\end{equation}
Introducing the random marginal $BX(X^\top VX - I_p)$, we further decompose
\begin{align}
	\mathrm{Var}\left[UX(X^\top VX - I_p)\right] &= \mathbb{E}_{U, V}\left[\|UX(X^\top VX - I_p) - BX(X^\top VX - I_p) \|^2\right] \\
    &\qquad + \mathbb{E}_{V}\left[\|BX(X^\top VX - I_p) - BX(X^\top BX - I_p) \|^2\right].
\end{align}
The first term in the above is upper bounded as 
\begin{align}
    \mathbb{E}_{U, V}\left[\|UX(X^\top VX - I_p) - BX(X^\top VX - I_p) \|^2\right] &\leq\mathbb{E}_{U, V}\left[\|U - B\|^2\|X(X^\top VX - I_p)\|_2^2\right]\\
    &=\sigma_B^2\mathbb{E}_V[\|X(X^\top VX - I_p)\|_2^2]\\
    &\leq \sigma_B^2 \frac{1+\varepsilon}{\beta_n}\mathbb{E}_V[\|X^\top VX - I_p\|_2^2]\\
    &\leq \sigma_B^2 \frac{1+\varepsilon}{\beta_n}\left(\sigma_B^2\frac{1+\varepsilon}{\beta_n} + \varepsilon^2\right),
\end{align}
where we used $\|X\|^2\leq  \frac{1+\varepsilon}{\beta_n}$, and we control $\mathbb{E}_V[\|X^\top VX - I_p\|_2^2] \leq \mathbb{E}_V[\|X^\top VX - I_p\|^2] = \mathbb{E}_V[\|X^\top (V - B)X\|^2] + \|XBX^\top - I_p\|^2 \leq \sigma_B^2\frac{1+\varepsilon}{\beta_n} + \varepsilon^2$.
The second term is controlled by 
\begin{align}
    \mathbb{E}_{V}\left[\|BX(X^\top VX - I_p) - BX(X^\top BX - I_p) \|^2\right] &=\mathbb{E}_{V}\left[\|BXX^\top (V-B)X \|^2\right]\\
    &\leq \sigma_B^2\|B\|_2^2\|X\|_2^6\\
    &\leq \sigma_B^2\beta_1^2\frac{(1+\varepsilon)^3}{\beta_n^3},
\end{align}
where we used  $\|X\|_2^2\leq  \frac{1+\varepsilon}{\beta_n}$ and $\|B\|_2 = \beta_1$.
Taking things together we obtain
\begin{align}
	\mathrm{Var}\left[UX(X^\top VX - I_p)\right] &\leq \sigma_B^2\left(\frac{1+\varepsilon}{\beta_n}\left( \sigma_B^2\frac{1+\varepsilon}{\beta_n} + \varepsilon^2\right) + \beta_1^2\frac{(1+\varepsilon)^3}{\beta_n^3}\right).
\end{align}

Similarly, the variance of the first term in the landing  is controlled by introducing yet another random variable $G$ that takes values from $\nabla f_\xi(X)$. 
We use the U-statistics variance decomposition twice to get
\begin{align*}
    \mathrm{Var}[\sk\left( G X^\top U\right) V X] &= \mathbb{E}_{G, U, V}[\|\sk((G - \nabla f(X))X^\top U)VX\|^2] \\
	&\qquad + \mathbb{E}_{U, V}[\|\sk(\nabla f(X)X^\top (U - B))VX\|^2] \\
	&\qquad+ \mathbb{E}_{V}[\|\sk(\nabla f(X)X^\top B)(V - B)X\|^2].
\end{align*}
The first term is upper bounded by doing
\begin{align}
    \mathbb{E}_{G, U, V}[\|\sk((G - \nabla f(X))X^\top U)VX\|^2] &\leq \mathbb{E}_{G, U, V}[\|G - \nabla f(X)\|^2\|X^\top U\|_2^2\|VX\|_2^2]\\
    &\leq \sigma_G^2 \mathbb{E}_U[\|U\|^2_2]^2\|X\|_2^4\\
    &\leq  \sigma_G^2 p_B^2 \frac{(1+\varepsilon)^2}{\beta_n^2},
\end{align}
where we used $p_B = \mathbb{E}_U[\|U\|^2_2] = \mathbb{E}_{B_\zeta}[\|B_\zeta\|^2_2]$.
The second term gives
\begin{align}
    \mathbb{E}_{U, V}[\|\sk(\nabla f(X)X^\top (U - B))VX\|^2] &\leq \mathbb{E}_{U, V}[\|\nabla f(X)X^\top \|_2^2\|U - B\|^2\|VX\|_2^2]\\
    &\leq \sigma_B^2 \|\nabla f(X)X^\top\|_2^2\mathbb{E}_U[\|U\|^2]\|X\|_2^2\\
    &\leq \sigma_B^2 \Delta p_B\frac{1+\varepsilon}{\beta_n},
\end{align}
where $\Delta$ upper-bounds $\|\nabla f(X)X^\top \|_2^2$.
The third term gives
\begin{align}
    \mathbb{E}_{V}[\|\sk(\nabla f(X)X^\top B)(V - B)X\|^2]&\leq \mathbb{E}_{V}[\|\nabla f(X)X^\top \|_2^2\|B\|^2_2\|V - B\|^2\|X\|_2^2]\\
    &\leq \sigma_B^2 \|\nabla f(X)X^\top\|^2_2\| B\|_2^2\|X\|_2^2 \\
    &\leq \sigma_B^2\Delta \beta_1^2 \frac{1+\varepsilon}{\beta_n},
\end{align}
which leads to the bound
\begin{align*}
    \mathrm{Var}[\sk\left( G X^\top U \right) V X]&\leq \sigma_G^2 p_B^2 \frac{(1+\varepsilon)^2}{\beta_n^2} + \sigma_B^2\frac{1+\varepsilon}{\beta_n}\Delta\left( p_B + \beta_1^2 \right).
\end{align*}
Finally, we join these two bounds using the generic inequality $\mathrm{Var}[a + b]\leq 2(\mathrm{Var}[a] + \mathrm{Var}[b])$, which gives
\begin{align}
     \mathbb{E}_\Xi[\|\tilde{E}(X, \Xi)\|^2]&=\mathrm{Var}[2\sk\left( G X^\top U \right) V X + 2\omega VX(X^\top UX - I_p)] \\
    &\leq 8(\mathrm{Var}[\sk\left( G X^\top U \right) V X] + \omega^2\mathrm{Var}[VX(X^\top UX - I_p)])\\
    &\leq 8\left(\sigma_G^2 p_B^2 \frac{(1+\varepsilon)^2}{\beta_n^2} + \sigma_B^2\frac{1+\varepsilon}{\beta_n}\Delta\left( p_B + \beta_1^2 \right) + \omega^2 \sigma_B^2\left(\frac{1+\varepsilon}{\beta_n}(\sigma_B^2\frac{1+\varepsilon}{\beta_n} + \varepsilon^2)+ \beta_1^2\frac{(1+\varepsilon)^3}{\beta_n^3}\right) \right)\\
    &=8\sigma_G^2p_B^2\frac{(1+\varepsilon)^2}{\beta_n^2} + 8\sigma_B^2\frac{1+\varepsilon}{\beta_n}\left(\Delta\left( p_B + \beta_1^2 \right) + \omega^2\left(\sigma_B^2\frac{1+\varepsilon}{\beta_n} + \varepsilon^2 + \beta_1^2\frac{(1+\varepsilon)^2}{\beta_n^2}\right)\right)\\
    &= \sigma_G^2 \alpha_G + \sigma_B^2(\alpha_B + \omega^2 \gamma_B),
\end{align}
with 
\begin{align}
     \alpha_G  &= 8p_B^2\frac{(1+\varepsilon)^2}{\beta_n^2} \\
     \alpha_B & = 8\frac{1+\varepsilon}{\beta_n}\Delta\left( p_B + \beta_1^2 \right) \\
     \gamma_B &= 8\frac{1+\varepsilon}{\beta_n}\left(\frac{1+\varepsilon}{\beta_n}\sigma_B^2 + \varepsilon^2 + \beta_1^2\frac{(1+\varepsilon)^3}{\beta_n^3}\right). \label{eq:variance_bound_constants}
\end{align}
\end{proof}

\section{Riemannian Interpretation of $\Psi_B^\mathrm{R}(X)$ in Proposition~\ref{prop:generalized_stiefel_relative_descent}}
\label{app:canonical_reldescent}
Similar to the work of \citet{Gao2022Optimization}, we provide a geometric interpretation of the relative ascent direction $\relgrad_B^\mathrm{R}(X)$ as a Riemannian gradient in a metric induced by an isometry 
\[
\Phi_{B,M}: \stiefelG{} \to \stiefelG{B,M}: Y \mapsto B^{-\frac12} Y M^{\frac12}
\]
between the standard Stiefel manifold $\stiefelG{}$ and the doubly generalized Stiefel manifold
\[
\stiefelG{B,M} := \{X: X^\top B X = M\},
\]
for $B,M\succ 0$, which is a layered manifold \citep{Goyens2024Computing} of $h(X) := X^\top B X$.

The map $\Phi_{B,M}$ extends to a diffeomorphism of the set of the full rank $\bR^{n\times p}$ matrices onto itself and maps the standard Stiefel manifold $\stiefelG{}$ to the generalized Stiefel manifold $\stiefelG{B,M}$. The tangent space at $X\in\stiefelG{B,M}$ is the null space of $\mathrm{D}h(X)$:
\begin{align}
    \mathrm{T}_{X}\stiefelG{B,M} &= \{\xi\in\bR^{n\times p}: \xi^T BX+X^TB\xi=0\} \nonumber\\
    &= \{X(X^TBX)^{-1}\Omega+B^{-1}X_{\perp}K: \Omega^T+\Omega=0, \Omega\in\bR^{p\times p}, K\in\bR^{(n-p)\times p}\} \nonumber\\
    &= \{WBX: W^T+W=0, W\in\bR^{n\times n}\} \nonumber\\
    &= \{\Phi_{B,M}(\zeta):\zeta\in\mathrm{T}_{\Phi^{-1}_{B,M}(X)}\stiefelG{}\}, \nonumber
\end{align}
where $X_\perp\in\mathbb{R}^{n\times(n-p)}$ is any matrix such that 
$\mathrm{span}(X_\perp)$ is the orthogonal complement of $\mathrm{span}(X)$.

Consider the canonical metric on the standard Stiefel manifold $\stiefelG{}$:
\[
g_{Y}^{\stiefelG{}}(Z_1, Z_2) = \inner{Z_1}{(I-\frac12 YY^T)Z_2}.
\]
It turns out that the Riemannian gradient of $\tilde{f}:\stiefelG{} \to \mathbb{R}$ is 
\[
\grad \tilde{f}(Y) = 2\,\sk\left(\nabla \tilde{f}(Y) Y^\top\right) Y.
\]

Using the map $\Phi_{B,M}$, we define the metric $g^{\stiefelG{B,M}}$ which makes $\Phi_{B,M}$ an isometry. This metric is given by
\begin{align*}
    g^{\stiefelG{B,M}}_X (\xi,\zeta) &= g^{\stiefelG{}}_{\Phi_{B,M}^{-1}(X)} (\Phi_{B,M}^{-1}(\xi), \Phi_{B,M}^{-1}(\zeta)) \\
    &= \inner{\xi}{(B - \frac12 B X (X^TBX)^{-1} X^T B) \zeta (X^T B X)^{-1}}. 
\end{align*}
This metric extends to arguments $\xi$ and $\zeta$ in $\mathrm{T}_X\mathbb{R}^{n\times p} \simeq \mathbb{R}^{n\times p}$ using the same formula. With respect to this metric, the normal space of $\stiefelG{B,M}$ is  
\begin{align*}
    \mathrm{N}_{X}\stiefelG{B,M} &= \{X(X^TBX)^{-1}S: S^T=S, S\in\bR^{p\times p}\}.
\end{align*}

The form of the derived tangent and normal spaces allow us to derive their projection operators $P_X$ and $P_X^{\perp}$ respectively as
\begin{align*}
    P_X^{\perp}(Y) &= X(X^TBX)^{-1}\sym(X^TBY), \\
    P_X(Y) &= Y - X(X^TBX)^{-1}\sym(X^TBY).
\end{align*}

Since $\Phi_{B,M}$ is a linear isometric, and letting $\Phi_{B,M}^*$ denote the adjoint of $\Phi_{B,M}$ with respect to the Frobenius inner product, the Riemannian gradient w.r.t. $g^{\stiefelG{B,M}}$ can be computed directly by 
\begin{align*}
     \grad_{B,M} f(X) &= \Phi_{B,M} \left(\grad (f\circ\Phi_{B,M})(\Phi_{B,M}^{-1}(X))\right) \\
     &= \Phi_{B,M} \left(2\,\sk\left(\nabla (f\circ\Phi_{B,M})(\Phi_{B,M}^{-1}(X))\right) (\Phi_{B,M}^{-1}(X))^\top\right) \\
     &= \Phi_{B,M} \left(2\,\sk\left(\Phi_{B,M}^* \nabla f(X) (\Phi_{B,M}^{-1}(X))^\top\right) (\Phi_{B,M}^{-1}(X))^\top\right) \\
     &= 2B^{-\frac12} \sk\left(B^{-\frac12} \nabla f(X) M^{\frac12} (B^{\frac12}XM^{-\frac12})^\top\right) B^{\frac12}XM^{-\frac12} M^{\frac12}\\
     &= 2\,\sk(B^{-1}\nabla f(X) X^\top)BX,
\end{align*}
where the second and fourth equalities follow from the Riemannian gradient on the Stiefel manifold and the definition of $\Phi_{B,M}$, respectively. Alternatively, one can check that the obtained expression indeed satisfies the characteristic properties of the gradient, as was done in the proof of~\citet[Proposition~4]{Gao2022Optimization}.

%Hence, akin to the work of \citep{Gao2022Optimization} for the standard Stiefel manifold, we derived the equivalent Riemannian interpretation of $\relgrad_B^\mathrm{R}(X)$ and the landing algorithm for the generalized Stiefel manifold $\stiefelG{B}$. 
Note that the formula for $\relgrad_B^\mathrm{R}(X)$ involves computing an inverse of $B$ and thus does not allow a simple unbiased estimator to be used in the stochastic case, as opposed to $\relgrad_B(X)$.

\newpage \section{Errata with respect to the ICML 2024 version} 

\begin{enumerate}
    \item{Pages 2 and 7, \citep{Bonnabel2013Stochastic}: The rate of convergence of Riemannian SGD can be found instead in Theorem 5 (section B) of \citep{Zhang2016Riemannian}.}
    \item{Page 2, paragraph above \eqref{eq:optimization_cca}: The formulation follows from eq.\ (1) in \citep{Arora2017Stochastic}. Reformulated the sentence right above \eqref{eq:optimization_cca} to be more precise and in line with the description in the first paragraph of \citep{Arora2017Stochastic}.}
    \item{Page 5, \cref{ass:h_constraint}: The penalty term $\cN(x)$ is $L_\cN$-smooth, not its gradient $\nabla\cN(x)$.}
    \item{Page 7, \cref{prop:gen_stiefel_constants}: Added the value of  $L_\cN$, since it is also mentioned in \cref{ass:h_constraint}, and referred to the relevant lemma for its derivation.}
    \item{Page 7, \autoref{thm:landing_convergence}: replace $x_k$ by $x^k$.}
    \item{Page 7, \autoref{thm:landing_convergence}: In the second displayed equation we deleted $\omega^2$ and in the next line we inserted $\omega^2$ before $C_h^2$.}
    \item{Page 7, \autoref{thm:landing_stochastic_convergence}: We deleted the two occurrences of $\omega^2$ and in the last line of Theorem 2.9, we inserted $\omega^2$ before $C_h^{-2}$.}
    \item{Page 17, section C.5: We deleted the five occurrences of $\omega^2$ and inserted $\omega^2$ before the two occurrences of $C_h^2$ (in (19) and two lines below).}
    \item{Page 18, section C.6: Delete the seven occurrences of $\omega^2$. Then, insert $\omega^2$ before the two occurrences of $C_h^2$ (in the last line of the first displayed equation and in the line above (21)).}
    \item{Page 18, The third line of the first displayed equation: In ``$h(x)$'', replace $x$ by $x^k$.}
\end{enumerate}

\end{document}